\documentclass[12pt]{article}
\usepackage{fullpage}
\usepackage{natbib}
\usepackage[helvratio=0.86]{newtxtext}

\usepackage{tikz-cd}
\usetikzlibrary{positioning}
\usepgflibrary{shapes}

\usepackage{graphicx,adjustbox}
\usepackage{booktabs,longtable}
\usepackage{amsmath,amssymb,amsthm}
\usepackage{authblk}

\usepackage{hyperref}
\hypersetup{%
    colorlinks=true,
    linkcolor=darkgray,
    citecolor=darkgray,
   urlcolor  = gray}
\usepackage{url}

\newcommand*{\ldblbrace}{\{\mskip-5mu\{}
\newcommand*{\rdblbrace}{\}\mskip-5mu\}}

\newcommand{\ldbl}{\{\!\!\{}
\newcommand{\rdbl}{\}\!\!\}}

\newcommand{\Nb}{\mathbb{N}}

\newcommand{\wl}[1]{\mathsf{WL}}
\newcommand{\wlk}[1]{#1\text{-}\mathsf{WL}}

\newcommand{\MPNN}{\mathsf{MPNN}}
\newcommand{\MPNNs}{\mathsf{MPNN}\text{s}}
\newcommand{\GNN}{\mathsf{GNN}}
\newcommand{\GNNs}{\mathsf{GNN}\text{s}}

\newcommand{\GSN}{\mathsf{GSN}}
\newcommand{\GSNs}{\mathsf{GSN}\text{s}}

\newcommand{\homc}[1]{\mathsf{hom}(#1)}

\newtheorem{proposition}{Proposition}
\newtheorem{theorem}{Theorem}

\newtheorem{lemma}{Lemma}

\title{Graph Neural Networks with Local Graph Parameters}
\author[1]{Pablo Barcel\'o}
\author[2]{Floris Geerts}
\author[1]{Juan Reutter}
\author[2]{Maksimilian Ryschkov}
\affil[1]{PUC Chile and IMFD Chile}
\affil[2]{University of Antwerp}

\date{}
\begin{document}
\maketitle
	
\begin{abstract}
Various recent proposals increase the distinguishing power of Graph Neural Networks $(\GNNs)$ by propagating features between $k$-tuples of vertices. The distinguishing power of these ``higher-order'' $\GNNs$ is known to be bounded by the $k$-dimensional Weisfeiler-Leman ($\mathsf{WL}$) test, yet their $\mathcal O(n^k)$ memory requirements limit their applicability. Other proposals infuse $\GNNs$ with local higher-order graph structural information from the start, hereby inheriting the desirable $\mathcal O(n)$ memory requirement from $\GNNs$ at the cost of a one-time, possibly non-linear, preprocessing step. We propose local graph parameter enabled $\GNNs$ as a framework for studying the latter kind of approaches and precisely characterize their distinguishing power, in terms of a variant of the $\mathsf{WL}$ test, and in terms of the graph structural properties that they can take into account. Local graph parameters can be added to any $\GNN$ architecture, and are cheap to compute. 
In terms of expressive power, our proposal lies in the middle of $\GNNs$ and their higher-order counterparts. Further, we propose several techniques to aide in choosing the right local graph parameters. Our results connect $\GNNs$ with deep results in finite model theory and finite variable logics.
Our experimental evaluation shows that adding local graph parameters often has a positive effect for a variety of $\GNNs$, datasets and graph learning tasks.\looseness=-1
\end{abstract}	

\section{Introduction}\label{sec:intro}

\paragraph{Context.} 
Graph neural networks ($\GNNs$) \citep{MerkwirthL05,scarselli2008graph}, and its important class of Message Passing Neural Networks ($\MPNNs$) \citep{GilmerSRVD17}, are one of the most popular methods for graph learning tasks. Such $\MPNNs$ use an iterative message passing scheme, based on the adjacency structure of the underlying graph, to compute vertex (and graph) embeddings in some real Euclidean space.
The expressive (or discriminative) power of $\MPNNs$ is, however, rather limited \citep{xhlj19,grohewl}. Indeed, $\MPNNs$ will always identically embed two vertices (graphs) when these vertices (graphs) cannot be distinguished by the one-dimensional Weisfeiler-Leman ($\wl{}$) algorithm. Two graphs $G_1$ and $H_1$ and vertices $v$ and $w$ that cannot be distinguished by $\wl{}$ (and thus any $\MPNN$) are shown in Fig.~\ref{fig:wlequivalent}. The expressive power of $\wl{}$ is well-understood \citep{CaiFI92,ARVIND202042,DellGR18} and basically can only use \textit{tree-based} structural information in the graphs to distinguish vertices. As a consequence, no $\MPNN$ can detect that vertex $v$ in Fig.~\ref{fig:wlequivalent}
is part of a $3$-clique, whereas $w$ is not. Similarly, $\MPNNs$ cannot detect that $w$ is part of a $4$-cycle, whereas $v$ is not. Further limitations of $\wl{}$ in terms of graph properties can be found, e.g., in \citet{Furer17,ARVIND202042,chen2020graph,tahmasebi2020counting}.

To remedy the weak expressive power of $\MPNNs$, so-called \textit{higher-order} $\MPNNs$  
were proposed \citep{grohewl,maron2018invariant,Morris2020WeisfeilerAL}, whose expressive power is measured in terms of the $k$-dimensional $\wl{}$ procedures ($\wlk{k}$)  \citep{DBLP:conf/nips/MaronBSL19,chen2019powerful,azizian2020characterizing,geerts2020expressive,Sato2020ASO,damke2020novel}. In a nutshell, $\wlk{k}$ operates on $k$-tuples of vertices and allows to distinguish vertices (graphs) based on structural information related to \textit{graphs of treewidth $k$}  \citep{dvorak,DellGR18}. 
By definition, $\wl{} = \wlk{1}$. 
As an example, $\wlk{2}$ can detect that vertex $v$ in Fig.~\ref{fig:wlequivalent} belongs to a $3$-clique or a $4$-cycle since both have treewidth two. While 
more expressive than $\wl{}$, the $\GNNs$ based on $\wlk{k}$ require $\mathcal{O}(n^k)$ operations in \textit{each iteration}, where $n$ is the number of vertices, hereby hampering their applicability. 

\begin{figure}[t]
    \centering
\begin{tikzpicture}[main/.style = {draw, circle , minimum size=2.6mm,line width=1pt,node distance=0.6cm},main2/.style = {draw, circle , minimum size=4mm,line width=1pt,fill=gray!20,node distance=0.8cm}] 
\node[main,label=above:{$v$},label=right:{$(2)$}] (1) {}; 
\node[main,label=left:{$(2)$}] (2) [below left of=1] {}; 
\node[main,label=right:{$(2)$}] (3) [below right of=1]{};
\node[main,label=left:{$(2)$}] (4) [below of=2] {};
\node[main,label=right:{$(2)$}] (5) [below of=3] {};
\node[main,label=right:{$(2)$}] (6) [below right of=4] {};
\node (7) [below=0cm and 0.7cm of 6] {$G_1$};
\draw[line width=1pt] (1) -- (2);
\draw[line width=1pt] (1) -- (3);
\draw[line width=1pt] (2) -- (3);
\draw[line width=1pt] (3) -- (4);
\draw[line width=1pt] (4) -- (5);
\draw[line width=1pt] (4) -- (6);
\draw[line width=1pt] (5) -- (6);
\end{tikzpicture} \hspace{1cm}
\begin{tikzpicture}[main/.style = {draw, circle , minimum size=2.6mm,line width=1pt,node distance=0.6cm},main2/.style = {draw, circle , minimum size=4mm,line width=1pt,fill=gray!20,node distance=0.8cm}]
\node[main,label=above:{$w$},label=right:{$(0)$}] (1) {}; 
\node[main,label=left:{$(0)$}] (2) [below left of=1] {}; 
\node[main,label=right:{$(0)$}] (3) [below right of=1]{};
\node[main,label=left:{$(0)$}] (4) [below of=2] {};
\node[main,label=right:{$(0)$}] (5) [below of=3] {};
\node[main,label=right:{$(0)$}] (6) [below right of=4] {};
\node (7) [below=0cm and 0.7cm of 6] {$H_1$};
\draw[line width=1pt] (1) -- (2);
\draw[line width=1pt] (1) -- (3);
\draw[line width=1pt] (2) -- (4);
\draw[line width=1pt] (3) -- (5);
\draw[line width=1pt] (3) -- (4);
\draw[line width=1pt] (4) -- (6);
\draw[line width=1pt] (5) -- (6);
\end{tikzpicture} 
    \caption{Two graphs that are indistinguishable by the $\wl{}$-test. The numbers between round brackets indicate how many homomorphic images of the $3$-clique each vertex is involved in.}
    \label{fig:wlequivalent}
\end{figure}
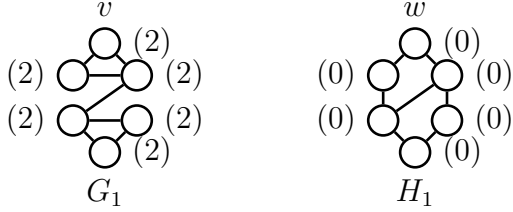

A more practical approach is to extend the expressive power of $\MPNNs$ \textit{whilst preserving their $\mathcal{O}(n)$ cost in each iteration}. Various such extensions  \citep{kipf-loose,chen2019powerful,li2019hierarchy,ishiguro2020weisfeilerlehman,bouritsas2020improving,geerts2020lets}  achieve this by infusing $\MPNNs$ with \textit{local graph structural information from the start}. That is, the iterative message passing scheme of $\MPNNs$ is run on vertex labels that contain quantitative information about local graph structures. It is easy to see that such architectures can go beyond the $\wl{}$ test: for example, 
adding triangle counts to $\MPNNs$ suffices to distinguish the vertices $v$ and $w$ and graphs $G_1$ and $H_1$ in Fig.~\ref{fig:wlequivalent}. Moreover, the cost is 
 \textit{a single preprocessing step} to count local graph parameters, thus maintaining the  $\mathcal{O}(n)$ cost in the iterations of the $\MPNN$. 
While there are some partial results showing that local graph parameters increase expressive power \citep{bouritsas2020improving,li2019hierarchy},  their precise expressive power and relationship to higher-order $\MPNNs$ was unknown, and  there is little guidance in terms of which local parameters do help $\MPNNs$ and which ones do not. 
The main contribution of this paper is a precise characterization of  the expressive power of $\MPNNs$ with local graph parameters and its relationship to the hierarchy of higher-order 
$\MPNNs$.

\paragraph{Our contributions.} 
In order to nicely formalize local graph parameters, we propose to extend vertex labels with \textit{homomorphism counts} of small graph patterns.\footnote{We recall that homomorphisms are edge-preserving mappings between the vertex sets.}  More precisely, given a graphs $P$ and $G$, and vertices $r$ in $P$ and $v$ in $G$, we add the number of homomorphisms from $P$ to $G$ that map $r$ to $v$, denoted by $\homc{P^r,G^v}$, to the initial features of $v$. 
Such counts satisfy conditions (i) and (ii). Indeed, homomorphism counts are known to \textit{measure the similarity} of vertices and graphs \citep{MR214529,Grohe20}, and  
serve as a \textit{basis for the efficient computation} of a number of other important graph parameters, e.g., subgraph and induced subgraphs counts \citep{Curticapean_2017,ZhangY0ZC20}. Furthermore, homomorphism counts underly \textit{characterisations of the expressive power} of $\MPNNs$ and higher-order $\MPNNs$. As an example, two vertices $v$ and $w$ in graphs $G$ and $H$, respectively, are indistinguishable by $\wl{}$, and hence by $\MPNNs$, precisely      when $\homc{T^r,G^v} = \homc{T^r,H^w}$ for every rooted tree $T^r$ \citep{dvorak,DellGR18}.

Concretely, we propose $\mathcal F$-$\MPNNs$ where $\mathcal F = \{P_1^r,\allowbreak \dots,\allowbreak P_\ell^r\}$ is a set of (graph) patterns,
by (i) first allowing a pre-processing step that labels each vertex $v$ of a graph $G$ with the vector $\bigl(\homc{P^r_1,G^v},\ldots,\homc{P^r_\ell,G^v}\bigr)$, and (ii) then run
an $\MPNN$ on this labelling. As such, we can turn \textit{any} $\MPNN$ into an $\mathcal F$-$\MPNN$
by simply augmenting the initial vertex embedding. Furthermore,
several recently proposed extensions
of $\MPNNs$ fit in this approach, including $\MPNNs$ extended with information about vertex degrees \citep{kipf-loose}, walk counts \citep{chen2019powerful}, tree-based counts \citep{ishiguro2020weisfeilerlehman} and subgraph counts \citep{bouritsas2020improving}. Hence, $\mathcal F$-$\MPNNs$ can also be regarded as a unifying theoretical formalism.

\smallskip
\noindent
Our main results can be summarised, as follows:

\begin{enumerate}
    \item We precisely characterise the expressive power of $\mathcal F$-$\MPNNs$ by means of an extension of $\wl{}$, denoted by $\mathcal F$-$\wl{}$. 
For doing so, we use {\em $\mathcal F$-pattern trees}, which are obtained from standard trees by joining an arbitrary number of copies of the patterns in $\mathcal F$ 
to each one of its vertices. 
Our result states 
that vertices $v$ and $w$ in graphs $G$ and $H$, respectively, are indistinguishable by $\mathcal F$-$\wl{}$, and hence by $\mathcal F$-$\MPNNs$, precisely
    when $\homc{T^r,G^v} = \homc{T^r,H^w}$ for every $\mathcal F$-pattern tree $T^r$.  This characterisation gracefully extends the characterisation for standard $\MPNNs$, mentioned earlier, by  setting $\mathcal F = \emptyset$.  Furthermore,
$\mathcal F$-$\MPNNs$ provide insights in the expressive power of existing $\MPNN$ extensions, most notably the Graph Structure Networks of \citet{bouritsas2020improving}.
	
    \item 	 We compare  $\mathcal F$-$\MPNNs$ to higher-order $\MPNNs$, which are characterized in terms of the $\wlk{k}$-test. 
    On the one hand, while $\mathcal F$-$\MPNNs$ strictly increase the expressive power of the $\wl{}$-test, for any finite set 
$\mathcal F$ of patterns, $\wlk{2}$ can distinguish graphs which $\mathcal F$-$\MPNNs$ cannot.
 On the other hand, for each $k \geq 1$ there are patterns $P$ 
such  that $\{P\}$-$\MPNNs$ can distinguish graphs which $\wlk{k}$ cannot.

    \item We deal with the technically challenging problem of pattern selection and comparing 
	$\mathcal F$-$\MPNNs$ based on the patterns included in  $\mathcal F$. We prove two partial results: one establishing when a pattern $P$ in $\mathcal F$ is redundant, based on whether or not $P$ is the join of other patterns in $\mathcal F$, and another result indicating when $P$ does add expressive power, based on the treewidth of $P$ compared to the treewidth of other patterns in $\mathcal F$.
	
	\item Our theoretical results are complemented by an experimental study in which we show that for various $\GNN$ architectures, datasets and graph learning tasks, all part of the recent benchmark by \citet{dwivedi2020benchmarkgnns}, the augmentation of initial features with homomorphism counts of graph patterns has often a positive effect, and the cost for computing these counts
incurs little to no overhead. 
\end{enumerate}
As such, we believe that $\mathcal F$-$\MPNNs$ not only provide an elegant theoretical framework
for understanding local graph parameter enabled $\MPNNs$, they are also a valuable alternative to higher-order $\MPNNs$ as way to increase the expressive power of $\MPNNs$.
In addition, and as will be explained in Section \ref{sec:lgp}, $\mathcal F$-$\MPNNs$ provide a unifying framework for understanding the expressive power of several other existing extensions of $\MPNNs$. Proofs of our results and further details on the relationship to existing approaches and experiments can be found in the appendix.

\paragraph{Related Work.}
Works related to the distinguishing power of the $\wl{}$-test, $\MPNNs$ and their higher-order variants
are cited throughout the paper. 
Beyond distinguishability, $\GNNs$ are analyzed in terms of universality and generalization properties \citep{DBLP:journals/corr/abs-1901-09342,Keriven,Chen2019,garg2020generalization}, local distributed
algorithms \citep{DBLP:conf/nips/SatoYK19,Loukas2020What}, randomness in features  \citep{sato2020random,abboud2020surprising} and using local context matrix features  \citep{vignac2020building}. 
Other extensions of $\GNNs$ are surveyed, e.g., in
\citet{Zonghan2019,Zhou2018} and \citet{chami2021machine}. 
Related are the Graph Homomorphism Convolutions by \citet{nt2020graph}  which apply $\mathsf{SVM}$s directly on the representation of vertices by homomorphism counts. 
Finally, our approach is reminiscent of the graph representations by means of graphlet kernels \citep{shervashidze09a},  but then on the level of vertices. 

\section{\boldmath Local Graph Parameter Enabled $\MPNNs$}\label{sec:lgp} 
In this section we introduce $\MPNNs$ with local graph parameters. We start by introducing preliminary concepts.

\paragraph{Graphs.}
We consider undirected vertex-labelled graphs $G=(V,E,\chi)$, with $V$ the set of vertices, $E$ 
the set of edges and $\chi$ a mapping assigning a label to each vertex in $V$. 
The set of neighbours of a vertex is denoted as $N_G(v) = \bigl\{u \in V \bigm\vert \{u,v\}\in E \bigr\}$.
A \textit{rooted graph} is a graph in which one its vertices is declared as its root.
We denote a rooted graph by $G^v$, where $v\in V$ is the root and depict them as graphs in which the root is a blackened vertex. 
Given graphs $G = (V_G,E_G,\chi_G)$ and $H = (H_H,E_H,\chi_H)$, an {\em homomorphism} $h$ is a mapping $h:V_G \to V_H$ such that (i) $\{h(u),h(v)\}\in E_H$ for every $\{u,v\}\in E_G$, and (ii) 
$\chi_G(u) = \chi_H(h(u))$ for every $u \in V_G$. 
For rooted graphs $G^v$ and $H^w$, an homomorphism must additionally map $v$ to $w$. 
We denote by 
$\homc{G,H}$ the number of homomorphisms from $G$ to $H$; similarly for rooted graphs. For simplicity of exposition we focus on vertex-labelled undirected graphs but all our results can be extended to edge-labelled directed graphs.

\paragraph{Message passing neural networks.}
The basic architecture for $\MPNNs$ \citep{GilmerSRVD17}, 
and the one used in recent studies on 
GNN expressibility \citep{grohewl,xhlj19,barcelo2019logical}, consists of a sequence of rounds that update the feature vector of every vertex in the graph by combining its current feature vector with the result of an aggregation over the feature vectors of its neighbours.  Formally, for a graph $G=(V,E,\chi)$, let $\mathbf{x}_{M,G,v}^{(d)}$ denote the feature vector computed for vertex $v\in V$ by an $\MPNN$ $M$ in round $d$. The initial feature vector $\mathbf{x}_{M,G,v}^{(0)}$ is a one-hot encoding of its label $\chi(v)$.  
This feature vector is iteratively updated in a number of rounds. In particular, in round $d$,
$$
\mathbf{x}_{M,G,v}^{(d)} \, := \, \textsc{Upd}^{(d)}\Big(\mathbf{x}_{M,G,v}^{(d-1)}, 
\textsc{Comb}^{(d)} \big( \ldblbrace \mathbf{x}_{M,G,u}^{(d-1)}
\, \mid \, u\in N_G(v) \rdblbrace \big)\Big),
$$
where $\textsc{Comb}^{(d)}$ and $\textsc{Upd}^{(d)}$ are an \textit{aggregating} and \textit{update} function, respectively. 
Thus, the feature vectors $\mathbf{x}_{M,G,u}^{(d-1)}$ of all neighbours $u$ of $v$ are combined by the aggregating function $\textsc{Comb}^{(d)}$ into a single vector, 
and then this vector is used together with $\mathbf{x}_{M,G,v}^{(d-1)}$ in order to produce $\mathbf{x}_{M,G,v}^{(d)}$ by applying the update function 
$\textsc{Upd}^{(d)}$.  

\paragraph{\boldmath $\MPNNs$ with local graph parameters.}
The $\GNNs$ studied in this paper leverage the power of $\MPNNs$ 
by enhancing initial features of vertices with \textit{local graph parameters} that are beyond their 
classification power. To illustrate the idea, consider the graphs in Fig.~\ref{fig:wlequivalent}. As mentioned, these graphs cannot be distinguished by the $\wl{}$-test,
 and therefore cannot be distinguished by the broad class of $\MPNNs$ (see e.g. \citep{xhlj19,grohewl}). 
If we allow a \textit{pre-processing stage}, however, in which the initial labelling of every vertex $v$ is extended with 
the number of (homomorphic images of) $3$-cliques in which $v$  participates (indicated by numbers between brackets in Fig.~\ref{fig:wlequivalent}), then clearly 
vertices $v$ and $w$ (and the graphs $G_1$ and $H_1$) can be distinguished based on this extra structural information. In fact, the initial labelling already suffices for this purpose. 

Let $\mathcal F = \{P_1^r, \dots,P_\ell^r\}$ be a set of (rooted) graphs, which we refer to as \emph{patterns}. Then, \textit{$\mathcal F$-enabled $\MPNNs$}, or just $\mathcal F$-$\MPNNs$, 
are defined in the same way as $\MPNNs$ with the crucial difference that now  
the initial feature vector of a vertex $v$ is a one-hot encoding of the label $\chi_G(v)$ of the vertex, and all the homomorphism counts from patterns in $\mathcal F$. Formally, 
in each round $d$ an $\mathcal F$-$\MPNN$ $M$ labels each vertex $v$ in graph $G$ with a feature vector $\mathbf{x}_{M,\mathcal F, G,v}^{(d)}$ which is inductively defined as follows: 
\begin{align*} \mathbf{x}_{M,\mathcal F, G,v}^{(0)} & := \big(\chi_G(v), \homc{P_1^r,G^v},\dots,\homc{P_\ell^r,G^v} \big)  \label{eq:extend} \\
\mathbf{x}_{M,\mathcal F, G,v}^{(d)}  & := \, \textsc{Upd}^{(d)}\Big(\mathbf{x}_{M,\mathcal F, G,v}^{(d-1)},  
\textsc{Comb}^{(d)} \big( \ldblbrace \mathbf{x}_{M,\mathcal F, G,v}^{(d-1)}
\, \mid \, u\in N_G(v) \rdblbrace \big)\Big). 
\end{align*}
We note that standard $\MPNNs$ are $\mathcal F$-$\MPNNs$ with $\mathcal{F}=\emptyset$.
As for $\MPNNs$, we can equip $\mathcal F$-$\MPNNs$ with a \textsc{Readout} function that aggregates all final feature vectors into a single feature vector in order to classify or distinguish 
graphs.

We emphasise that any $\MPNN$ architecture can be turned in an $\mathcal{F}$-$\MPNN$ by a simple
homomorphism counting preprocessing step. As such, we propose a \textit{generic plug-in for a large class of $\GNN$ architectures}. Better still, homomorphism counts of small graph patterns can be efficiently computed even on large datasets \citep{ZhangY0ZC20} and they form the basis for counting (induced) subgraphs and other notions of subgraphs \citep{Curticapean_2017}. Despite the simplicity of $\mathcal F$-$\MPNNs$, we will show that they can substantially increase the power of $\MPNNs$ by varying $\mathcal F$, only paying the one-time cost for preprocessing. 

\paragraph{\boldmath $\mathcal F$-$\MPNNs$ as unifying framework.}
An important aspect of $\mathcal F$-$\MPNNs$ is that they \textit{allow a principled analysis of the power of existing extensions of $\MPNNs$}.
For example, taking $\mathcal{F}=\{\raisebox{-1.1\dp\strutbox}{\includegraphics[height=2.7ex]{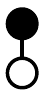}}\}$ suffices to capture
degree-aware $\MPNNs$ \citep{geerts2020lets}, such as the Graph Convolution Networks ($\mathsf{GCN}$s) \citep{kipf-loose}, which use the \textit{degree of vertices}; taking 
 $\mathcal F=\{L_1,L_2,\ldots,L_\ell\}$ for rooted paths $L_i$ of length $i$ suffices to model the \textit{walk counts} used in \citet{chen2019powerful}; and taking $\mathcal F$ as the set of labeled trees of depth one precisely corresponds to the use of the \textit{$\wl{}$-labelling obtained after one round} by  \citet{ishiguro2020weisfeilerlehman}. Furthermore, $\{C_\ell\}$-$\MPNNs$, where $C_\ell$ denotes the cycle of length $\ell$, correspond to the extension proposed in Section 4 in  \citet{li2019hierarchy}. 
 
In addition, $\mathcal F$-$\MPNNs$ can also capture the \textit{Graph Structure Networks} ($\GSNs$) by  \citet{bouritsas2020improving}, which use subgraph isomorphism counts of graph patterns. We recall that an isomorphism from $G$ to $H$ is a \textit{bijective} homomorphism $h$ from $G$ to $H$ which additionally satisfies (i) $\{h^{-1}(u),h^{-1}(v)\}\in E_G$ for every $\{u,v\}\in E_H$, and (ii) $\chi_G(h^{-1}(u)) = \chi_H(u)$ for every $u \in V_H$. When $G$ and $H$ are rooted graphs, isomorphisms should preserve the roots as well. 
Now, in a $\GSN$, the feature vector of each vertex $v$ is augmented with the the counts of every isomorphism from a rooted pattern $P^r$ to $G^v$, for rooted patterns $P$ in a set of patterns $\mathcal P$,\footnote{The use of orbits in \citet{bouritsas2020improving} is here ignored, but explained in the appendix.} and this is followed by the execution of an $\MPNN$, just as for our $\mathcal F$-$\MPNNs$.  It now remains  to observe that subgraph isomorphism counts can be computed in terms of homomorphism counts, and vice versa \citep{Curticapean_2017}. 
That is, $\GSNs$ can be viewed as $\mathcal F$-$\MPNNs$ and thus our results for  $\mathcal F$-$\MPNNs$ carry over to $\GSNs$. We adopt homomorphism counts instead of subgraph isomorphism counts because homomorphisms counts underly existing characterizations of the expressive power of $\MPNNs$ and homomorphism counts are in general more efficient to compute. Also, \citet{Curticapean_2017} indicate that all common graph counts are interchangeable in terms of expressive power.

\begin{figure}[t]
    \centering
\begin{tikzpicture}[main/.style = {draw, circle , minimum size=2.3mm,line width=1pt,node distance=0.8cm},main2/.style = {draw, circle , minimum size=4mm,line width=1pt,fill=gray!20,node distance=0.8cm}] 
\node[main] (1) {}; 
\node[main] (2) [below left of=1] {}; 
\node[main] (3) [below right of=1]{};
\node[main] (4) [below of=2] {};
\node[main] (6) [below of=3] {};
\node[main,label={[label distance=-1mm]above:{$v$}}] (5) [right=0.2cm of 4] {};
\node[main] (7) [below of=4] {};
\node[main] (8) [below of=6] {};
\node[main] (9) [below right of=7] {};
\node (10) [below=0cm of 9] {$G_2$};
\draw[line width=1pt] (1) -- (2);
\draw[line width=1pt] (1) -- (3);
\draw[line width=1pt] (2) -- (3);
\draw[line width=1pt] (2) -- (4);
\draw[line width=1pt] (3) -- (6);
\draw[line width=1pt] (4) -- (5);
\draw[line width=1pt] (4) -- (7);
\draw[line width=1pt] (5) -- (6);
\draw[line width=1pt] (6) -- (8);
\draw[line width=1pt] (7) -- (8);
\draw[line width=1pt] (7) -- (9);
\draw[line width=1pt] (8) -- (9);
\end{tikzpicture} \hspace{1cm}
\begin{tikzpicture}[main/.style = {draw, circle , minimum size=2.3mm,line width=1pt,node distance=0.8cm},main2/.style = {draw, circle , minimum size=4mm,line width=1pt,fill=gray!20,node distance=0.8cm}]
\node[main] (1) {}; 
\node[main] (2) [above left of=1] {}; 
\node[main] (3) [above right of=1]{};
\node[main] (4) [below left of=1] {};
\node[main,label={[label distance=-2mm]above left:{$w$}}] (5) [right=0.2cm of 4] {};
\node[main] (6) [below right of=1] {};
\node[main] (8) [below=0.2cm of 5] {};
\node[main] (7) [below left of=8] {};
\node[main] (9) [below right of=8] {};
\node (10) [below left=0cm and 0.04cm of 9] {$H_2$};
\draw[line width=1pt] (1) -- (2);
\draw[line width=1pt] (1) -- (3);
\draw[line width=1pt] (1) -- (5);
\draw[line width=1pt] (2) -- (3);
\draw[line width=1pt] (2) -- (4);
\draw[line width=1pt] (3) -- (6);
\draw[line width=1pt] (4) -- (7);
\draw[line width=1pt] (6) -- (9);
\draw[line width=1pt] (7) -- (9);
\draw[line width=1pt] (7) -- (8);
\draw[line width=1pt] (9) -- (8);
\draw[line width=1pt] (8) -- (5);
\end{tikzpicture} 
    \caption{Two graphs and vertices that cannot be distinguished by the $\wl{}$-test, and therefore by standard $\MPNNs$. When considering $\mathcal F$-$\MPNNs$, with $\mathcal{F}=\{\raisebox{-1.1\dp\strutbox}{\includegraphics[height=3ex]{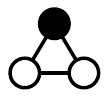}}\}$, one can show that both graphs cannot be distinguished just by focusing on the initial labelling, but they can be distinguished by an $\mathcal{F}$-$\MPNN$ with just one aggregation layer.  
    }	
    \label{fig:fwl_power_rounds}
\end{figure}

\section{\boldmath Expressive Power of $\mathcal{F}$-$\MPNNs$}\label{sec:exp-power}
Recall that the standard $\wl{}$-test \citep{Weisfeiler1968,grohe_2017} iteratively constructs a labelling of the vertices in a graph $G = (V,E,\chi)$ 
as follows. In round $d$, for each  vertex $v$ the algorithm first collects the label of $v$ and all of its neighbours after round $d-1$, 
and then it hashes this aggregated multiset 
of labels into a new label for $v$. The initial label of $v$ is  $\chi(v)$. 
As shown in \citet{xhlj19} and \citet{grohewl}, the $\wl{}$-test provides a bound on the classification power of $\MPNNs$: if two vertices or two graphs are indistinguishable by the 
$\wl{}$ test, then they will not be distinguished by any $\MPNN$. 

In turn, the expressive power of the $\wl{}$-test, and thus of $\MPNNs$, can be characterised in terms of homomorphism counts of trees \citep{dvorak,DellGR18}. 
This result can be seen as a characterisation of the expressiveness of the $\wl{}$-test in terms of a particular infinite-dimensional 
graph kernel: the one defined by 
the number of homomorphisms from every tree $T$ into the underlying graph $G$. 

In this section we show that both characterisations extend in an elegant way to the setting of $\mathcal{F}$-$\MPNNs$, confirming that 
$\mathcal{F}$-$\MPNNs$ are not just a useful, but also a well-behaved generalisation of standard $\MPNNs$.  
\subsection{\boldmath Characterisation in terms of $\wlk{\mathcal{F}}$}
We bound the expressive power of $\mathcal{F}$-$\MPNNs$ in terms of what we call the  {\em $\wlk{\mathcal{F}}$-test}. 
Formally, the $\wlk{\mathcal{F}}$-test extends $\wl{1}$ in the same way as $\mathcal{F}$-$\MPNNs$ extend standard $\MPNNs$: 
by including homomorphism counts of patterns in $\mathcal F$ in the initial labelling. 
That is, let $\mathcal{F}=\{P_1^r,\dots,P_\ell^r\}$. The $\wlk{\mathcal{F}}$-test is a vertex labelling algorithm 
that iteratively computes a label $\chi_{\mathcal{F},G,v}^{(d)}$ for each vertex $v$ of a graph $G$, defined as follows. 
\begin{align*} 
\chi_{\mathcal{F},G,v}^{(0)} \, & := \, \bigl(\chi_G(v),\homc{P_1^r,G^v},\dots,\homc{P_\ell^r,G^v}\bigr) \\
\chi_{\mathcal{F},G,v}^{(d)} & := \, \textsc{Hash}\bigl(\chi_{\mathcal{F},G,v}^{(d-1)},\ldbl \chi_{\mathcal{F},G,u}^{(d-1)}\mid u\in N_G(v) \rdbl\bigr).
\end{align*} 
The $\wlk{\mathcal{F}}$-test stops in round $d$ when no new pair of vertices are identified by means of $\chi_{\mathcal{F},G,v}^{(d)}$, that is, 
when for any two  vertices $v_1$ and $v_2$ from $G$,  $\chi_{\mathcal{F},G,v_1}^{(d-1)} = \chi_{\mathcal{F},G,v_2}^{(d-1)}$ implies 
$\chi_{\mathcal{F},G,v_1}^{(d)} = \chi_{\mathcal{F},G,v_2}^{(d)}$.
Notice that $\wl{1} = \wlk{\emptyset}$. 

We can use the $\wlk{\mathcal{F}}$-test to compare vertices of the same graphs, or different graphs. We say that the 
$\wlk{\mathcal{F}}$-test \emph{cannot distinguish vertices} 
if their final labels are the same, and 
that the $\wlk{\mathcal{F}}$-test \emph{cannot distinguish graphs} $G$ and $H$ if the multiset containing each label computed for $G$ is the same 
as that of $H$. 

Similarly as for $\MPNNs$ and the $\wl{}$-test \citep{xhlj19,grohewl}, we obtain that the  $\wlk{\mathcal{F}}$-test provides an upper bound 
for the expressive power of $\mathcal{F}$-$\MPNNs$.

\begin{proposition}\label{prop:wlupper}
If two vertices of a graph cannot be distinguished by the $\wlk{\mathcal{F}}$-test, then 
they cannot be distinguished by any $\mathcal F$-$\MPNN$ either. Moreover, 
if two graphs cannot be distinguished by the $\wlk{\mathcal{F}}$-test, then they cannot be distinguished by any  
$\mathcal F$-$\MPNN$ either.
\end{proposition}

We can also construct  $\mathcal F$-$\MPNNs$ that mimic the $\wlk{\mathcal F}$-test: Simply adding local parameters from a set 
$\mathcal F$ of patterns to the $\mathsf{GIN}$ architecture of \citet{xhlj19} results in an $\mathcal F$-$\MPNN$ that classifies vertices and graphs 
as 
the $\wlk{\mathcal{F}}$-test.

\subsection{\boldmath Characterisation in terms of $\mathcal{F}$-pattern trees}
At the core of several results about the $\wl{}$-test lies a characterisation linking the test with homomorphism counts of (rooted) trees \citep{dvorak,DellGR18}. In view of the connection to $\MPNNs$, it tells that $\MPNNs$ only use \textit{quantitative tree-based structural information from the underlying graphs}.
We next extend this characterisation to $\wlk{\mathcal{F}}$ by using homomorphism counts of so-called $\mathcal{F}$-pattern trees. 
In view of the connection with $\mathcal F$-$\MPNNs$ (Proposition~\ref{prop:wlupper}), this reveals that $\mathcal F$-$\MPNNs$ can use
quantitative information of \textit{richer graph structures} than $\MPNNs$.

To define $\mathcal{F}$-pattern trees we need the {\em graph join operator} $\star$. Given two rooted graphs $G^v$ and $H^w$, the join graph $(G\star H)^v$ is obtained by taking the disjoint union of $G^v$ and $H^w$, followed by identifying $w$ with  $v$. The root of the join graph is $v$. For example,
the join of \raisebox{-1.1\dp\strutbox}{\includegraphics[height=3ex]{K3.pdf}} and \raisebox{-1.1\dp\strutbox}{\includegraphics[height=3ex]{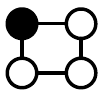}} is 
\raisebox{-1.4\dp\strutbox}{\includegraphics[height=4ex]{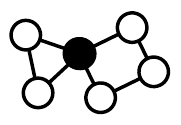}}. Further, if $G$ is a graph and $P^r$ is a rooted graph, then joining a vertex $v$ in $G$ with $P^r$ results 
in the disjoint union of $G$ and $P^r$, where $r$ is identified with $v$.

Let $\mathcal{F} = \{P_1^r,\ldots,P_\ell^r\}$. An \textit{$\mathcal{F}$-pattern tree} $T^r$ is obtained from a standard rooted tree $S^r=(V,E,\chi)$, called the \textit{backbone} of $T^r$, followed by joining every vertex $s\in V$ with any number of copies of patterns from $\mathcal{F}$. 

We define the \textit{depth} of an $\mathcal{F}$-pattern tree as the depth of its backbone.
Examples of $\mathcal{F}$-pattern trees, for $\mathcal{F}=\{\raisebox{-1.1\dp\strutbox}{\includegraphics[height=3ex]{K3.pdf}}\}$, are:
\begin{center}
\vspace{-1ex}
\includegraphics[height=1.7cm]{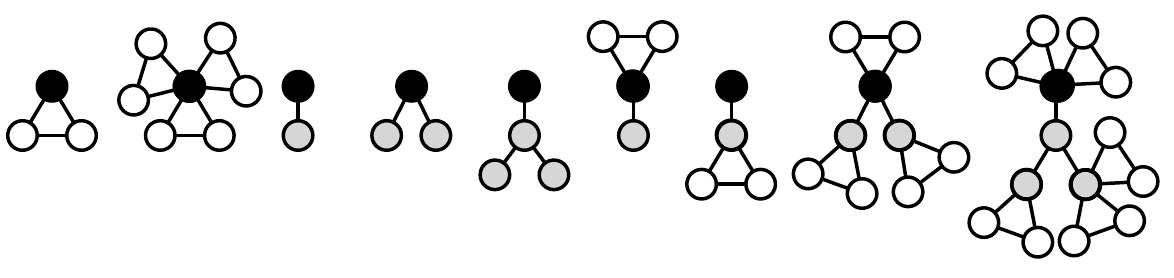}
\vspace{-2ex}
\end{center}
where grey vertices are part of the backbones of the $\mathcal{F}$-pattern trees.
Standard trees are also $\mathcal F$-pattern trees.

We next use $\mathcal{F}$-pattern trees to characterise the expressive power of $\wlk{\mathcal{F}}$ and thus, by Proposition~\ref{prop:wlupper}, of $\mathcal{F}$-$\MPNNs$.
\begin{theorem}\label{thm:char}
For any finite collection $\mathcal{F}$ of patterns, vertices $v$ and $w$ in a graph $G$  
 $\mathsf{hom}(T^r,G^v)=\mathsf{hom}(T^r,G^w)$ for every rooted $\mathcal{F}$-pattern tree $T^r$. 
Similarly, 
$G$ and $H$ are undistinguishable by the  $\wlk{\mathcal{F}}$-test if and only if 
$\mathsf{hom}(T,G)=\mathsf{hom}(T,H)$, for every (unrooted) $\mathcal{F}$-pattern tree.
\end{theorem}
The proof of this theorem, which can be found in the appendix, 
requires extending techniques from \citet{Grohe20,Grohe20Lics} that were used to characterise the expressiveness of 
$\wl{1}$ in terms of homomorphism counts of trees.

In fact, we can make the above theorem more precise. When $\wlk{\mathcal F}$ is run  for $d$ rounds, then \textit{only $\mathcal F$-patterns trees of depth $d$ are required}.
This tells that increasing the number of rounds of $\wlk{\mathcal{F}}$ 
results in that more complicated structural information is taken into account 
For example, consider the two graphs $G_2$ and $H_2$ and  vertices $v\in V_{G_2}$ and $w\in V_{H_2}$, shown in Fig.~\ref{fig:fwl_power_rounds}. Let $\mathcal{F}=\{\raisebox{-1.1\dp\strutbox}{\includegraphics[height=3ex]{K3.pdf}}\}$.
By definition, $\wlk{\mathcal F}$ cannot distinguish $v$ from $w$ based on the initial labelling. If run for one round, Theorem~\ref{thm:char} implies that $\wlk{\mathcal F}$ cannot distinguish $v$ from $w$ if and only if   $\mathsf{hom}(T^r,G_2^v)=\mathsf{hom}(T^r,H_2^w)$ for any $\mathcal{F}$-pattern tree of depth at most $1$.
It is  readily verified that 
$$\mathsf{hom}(\raisebox{-1.8\dp\strutbox}{\includegraphics[height=4ex]{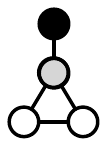}},G_2^v)=0\neq 4=\mathsf{hom}(\raisebox{-1.8\dp\strutbox}{\includegraphics[height=4ex]{Ftree1g.pdf}},H_2^w),$$
and thus $\wlk{\mathcal F}$ distinguishes $v$ from $w$ after one round. Similarly, $G_2$
and $H_2$ can be distinguished by $\wlk{\mathcal F}$ after one round.  We observe that
$G_2$ and $H_2$ are indistinguishable by $\wl{}$. Hence, $\mathcal F$-$\MPNNs$ are more expressive than $\MPNNs$.

Importantly, Theorem~\ref{thm:char} discloses the boundaries of $\mathcal F$-$\MPNNs$.  To illustrate this for some specific instances of $\mathcal F$-$\MPNNs$ mentioned earlier, the expressive power of degree-based $\MPNNs$ \citep{kipf-loose,geerts2020lets} is captured by $\{L_1\}$-pattern trees, and walk counts-$\MPNNs$ \citep{chen2019powerful} are captured by $\{L_1,\ldots, L_\ell\}$-pattern trees. These pattern trees are just trees, since joining paths to trees only results in bigger trees. Thus, 
Theorem~\ref{thm:char} tells that all these extensions are still bounded by $\wl{}$ (albeit needing less rounds). In contrast, beyond $\wl{}$,  $\{C_\ell\}$-pattern trees capture cycle count $\MPNNs$ \citep{li2019hierarchy}, and 
for $\GSNs$  \citep{bouritsas2020improving} which use subgraph isomorphism counts of pattern $P\in\mathcal P$, their expressive power is   captured by $\mathsf{spasm}(\mathcal P)$-pattern trees, where $\mathsf{spasm}(\mathcal P)$ consists of all surjective homomorphic images of  patterns in $\mathcal P$ \citep{Curticapean_2017}.

\section{\boldmath A Comparison with the $\wlk{k}$-test}\label{sec:choice}
We propose $\mathcal F$-$\MPNNs$ as an alternative and efficient way to extend the expressive power of
$\MPNNs$ (and thus the $\wl{}$-test) compared to the computationally intensive higher-order $\MPNNs$ based on the $\wlk{k}$-test \citep{grohewl,Morris2020WeisfeilerAL,DBLP:conf/nips/MaronBSL19}. In this section we 
situate $\wlk{\mathcal F}$ in the $\wlk{k}$ hierarchy.  The definition of $\wlk{k}$
is deferred to the appendix.

We have seen that $\wlk{\mathcal F}$ can distinguish graphs that $\wl{}$ cannot: it suffices to consider $\wlk{\{K_3\}}$ for the $3$-clique $K_3$.
In order to generalise this observation we need some notation. 
Let $\mathcal F$ and $\mathcal G$ be two sets of patterns and consider
an $\mathcal F$-$\MPNN$ $M$ and a $\mathcal G$-$\MPNN$ $N$. We say that $M$ is \textit{upper bounded in expressive power} by $N$ if for any graph $G$, 
if $N$ cannot distinguish vertices $v$ and $w$,\footnote{Just as for the $\wlk{\mathcal F}$-test, an $\mathcal F$-$\MPNN$ cannot distinguish two vertices if the label computed for 
both of them is the same} then neither can $M$. 
A similar notion is in place for pairs of graphs: if $N$ cannot distinguish graphs $G$ and $H$, then neither can $M$. 

More generally, let $\mathcal M$ be a class of  $\mathcal F$-$\MPNN$ and $\mathcal N$ be a class of  $\mathcal G$-$\MPNN$. We say that the class $\mathcal M$ is \textit{upper bounded in expressive power} by $\mathcal N$ if every $M\in\mathcal M$ is upper bounded in expressive power by an $N\in\mathcal N$ (which may depend on $M$). When $\mathcal M$ is upper bounded by $\mathcal N$ and vice versa, then $\mathcal M$ and $\mathcal N$ are said to have the \textit{same expressive power}. A class $\mathcal N$ is \textit{more expressive} than a class $\mathcal M$ when $\mathcal M$ is upper bounded in expressive power by $\mathcal N$, but there exist graphs that can be distinguished by $\MPNNs$ in $\mathcal N$ but not by any $\MPNN$ in $\mathcal M$.

Finally, we use the
notion of \textit{treewidth} of a graph, which measures the tree-likeness of a graph. For example,
trees have treewidth one, cycles have treewidth two, and the $k$-clique $K_k$ has treewidth $k-1$ (for $k>1$). We define this standard notion in the appendix and only note that we define the
treewidth of a pattern $P^r$ as the treewidth of its unrooted version $P$. 

Our first result is a consequence of the characterisation of $\wlk{k}$ in terms of homomorphism counts of graphs of treewidth $k$ \citep{dvorak,DellGR18}.
\begin{proposition}\label{prop:lgp-in-kwl}
For each finite set $\mathcal F$ of patterns, the expressive power of 
$\wlk{\mathcal{F}}$ is bounded by $\wlk{k}$, where $k$ is the largest
treewidth of a pattern in $\mathcal F$.
\end{proposition}
For example, since the treewidth of  $K_3$ is $2$, we  have that $\wlk{\{K_3\}}$ is bounded by $\wlk{2}$. Similarly, $\wlk{\{K_{k+1}\}}$ is bounded in expressive power by $\wlk{k}$. 

Our second result tells how to increase the expressive power of $\wlk{\mathcal F}$ beyond $\wlk{k}$. A pattern $P^r$ is a core if any homomorphism from $P$ to itself is injective. For example, any clique $K_k$ and cycle of odd length is a core.

\begin{theorem}\label{theo:games}
	Let $\mathcal F$ be a finite set of patterns.
If $\mathcal F$ contains a pattern $P^r$ which is a core and has treewidth $k$, then there exist graphs that can be distinguished by $\wlk{\mathcal F}$ but not by $\wlk{(k-1)}$.
\end{theorem}
In other words, for such $\mathcal F$,  $\wlk{\mathcal F}$ is not bounded by $\wlk{(k-1)}$.
For example, since $K_3$ is a core, $\wlk{\{K_3\}}$ is not bounded in expressive power by $\wl{}=\wlk{1}$.
More generally, $\wlk{\{K_{k}\}}$ is not bounded by $\wlk{(k-1)}$. 
The proof of Theorem \ref{theo:games} is based on extending deep 
techniques developed in finite 
model theory, and that have been used to understand the expressive power of {\em finite variable logics} \citep{AtseriasBD07,BovaC19}.  This result is stronger than the one underlying the strictness of the $\wlk{k}$ hierarchy \citep{O2017}, which states that $\wlk{k}$ is strictly more expressive than $\wlk{(k-1)}$. Indeed, \citet{O2017} only shows the \textit{existence} of a pattern $P^r$ of treewidth $k$ such that 
$\wlk{(k-1)}$ is not bounded by $\wlk{\{P^r\}}$. In Theorem \ref{theo:games} we provide an \textit{explicit recipe} for finding such a pattern $P^r$, that is, $P^r$ can be taken a core of treewidth $k$.

In summary, we have shown that there is a set $\mathcal F$ of patterns such that (i) $\wlk{\mathcal F}$ can distinguish graphs which cannot be distinguished by $\wlk{(k-1)}$, yet (ii) $\wlk{\mathcal F}$ cannot
distinguish more graphs than $\wlk{k}$. This begs the question whether there is a finite set $\mathcal F$ such that $\wlk{\mathcal F}$ is equivalent in expressive power to $\wlk{k}$.  We answer this negatively.
\begin{proposition}\label{prop:long-cycles}
For any $k>1$, there does not exist a finite set $\mathcal F$ of patterns such that 
$\wlk{\mathcal F}$ is equivalent in expressive power to $\wlk{k}$.
\end{proposition}
In view of the connection between $\mathcal F$-$\MPNNs$ and $\GSNs$ mentioned earlier, we thus show that no $\GSN$ can match the power of $\wlk{k}$, which was a question left open in  \citet{bouritsas2020improving}.
We remark that if we allow $\mathcal F$ to consist of all (\textit{infinitely many}) patterns of treewidth $k$, then $\wlk{\mathcal F}$ is equivalent in expressive power to $\wlk{k}$ \citep{dvorak,DellGR18}.

\section{When Do Patterns Extend Expressiveness?}\label{sec:patterns}
Patterns are not learned, but must be passed as an input to $\MPNNs$ together with the graph structure. Thus, knowing which patterns work well, and which do not, is of key importance for the power of the resulting $\mathcal{F}$-$\MPNNs$. 
This is a difficult question to answer
 since determining which patterns work well is clearly application-dependent. 
From a theoretical point of view, however, we can still look into interesting questions related to the problem of which patterns to choose. One such a question, and the one studied in this section, 
is 
when a pattern adds expressive power over the ones that we have already selected. More formally,  we study the following problem: 
Given a finite set $\mathcal F$ of patterns, 
when does adding a new pattern $P^r$ to $\mathcal F$ extends the expressive power of the $\wlk{\mathcal{F}}$-test?  

To answer this question in the positive, we need to find two graphs $G$ and $H$,  show that 
they are indistinguishable by the $\wlk{\mathcal{F}}$-test, but show that they can be distinguished by the 
$\wlk{\mathcal{F}\cup \{P^r\}}$-test.  As an example of this technique we show that 
 longer cycles always add expressive power. We use $C_k$ to represent the cycle of length $k$. 
\begin{proposition}\label{prop:cycles}
For any $k > 3$,  $\wlk{\{C^r_3,\dots,C^r_k\}}$ is more expressive than 
$\wlk{\{C^r_3,\dots,C^r_{k-1}\}}$.
\end{proposition}
We also observe that, by Proposition~\ref{prop:lgp-in-kwl}, $\wlk{\{C^r_3,\dots,C^r_k\}}$ is bounded by $\wlk{2}$ for any $k\geq 3$ because
cycles have treewidth two. It is often the case, however, that finding such graphs and proving that they are indistinguishable, can be rather challenging. 
Instead, in this section we provide two techniques that can be used 
to partially answer the question posed above by only looking at properties of the sets of patterns. 
Our first result is for establishing when a pattern does not add expressive power to a given set $\mathcal F$ of patterns, and the second one
when it does. 

\subsection{Detecting when patterns are superfluous} 
Our first result is a  simple recipe for choosing local features: instead of choosing complex patterns that are the joins of 
smaller patterns, one should opt for the smaller patterns. 

\begin{proposition}\label{prop:simplify}
Let $P^r = P_1^r \star P_2^r$ be a pattern that is the join of two smaller patterns. 
Then for any set $\mathcal F$ of patterns, we have that 
$\mathcal F \cup \{P^r\}$ is upper bounded by  $\mathcal F \cup \{P_1^r,P_2^r\}$.
\end{proposition}

Stated differently, this means that adding to $\mathcal F$ any pattern which is the join of two patterns already in 
$\mathcal F$ does not add expressive power. 

Thus, instead of using, for example, the pattern \raisebox{-2\dp\strutbox}{\includegraphics[height=4.6ex]{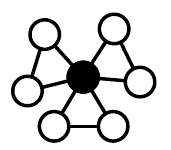}}, one should prefer to use instead the triangle \raisebox{-1.1\dp\strutbox}{\includegraphics[height=3ex]{K3.pdf}}. This result is in line with other advantages of smaller patterns: their homomorphism counts are easier to compute, and, since they are less specific, they should tend to produce less over-fitting. 

\subsection{Detecting when patterns add expressiveness} 
Joining patterns into new patterns does not give extra expressive power, but what about patterns which are not joins? 
We provide next a useful recipe for detecting when a pattern does add expressive power. 
We recall that the \textit{core of a graph} $P$ is its unique (up to isomorphism) induced subgraph which is a core.

\begin{theorem} \label{thm:increase}
Let $\mathcal{F}$ be a finite set of patterns and let
$Q^r$ be a pattern whose core has treewidth $k$. Then,
$\wlk{\mathcal F\cup \{Q^r\}}$ is more expressive than 
$\wlk{\mathcal F}$ if every pattern $P^r\in\mathcal F$
satisfies one of the following conditions: (i)~$P^r$  has treewidth $< k$; or (ii)~
$P^r$ does not map homomorphically to $Q^r$.
\end{theorem}
As an example, $\wlk{\{K_{3},\ldots, K_k\}}$ is more expressive than $\wlk{\{K_3,\ldots, K_{k-1}\}}$ for any $k>3$ because of the first condition. Similarly, $\wlk{\{ K_3,\ldots, K_k, C_\ell\}}$ is more expressive than $\wlk{\{K_3,\ldots,K_k\}}$ for odd cycles $C_\ell$. Indeed, such cycles are cores and no clique $K_k$ with $k>2$ maps homomorphically to $C_\ell$.

\section{Experiments \label{sec:experiments}}
We next showcase that $\GNN$
architectures benefit when  homomorphism counts of patterns are added
as additional vertex features. For patterns where homomorphism and subgraph
isomorphism counts differ (e.g., cycles) we compare with $\GSNs$ \citep{bouritsas2020improving}.
We use the benchmark for $\GNNs$ by \citet{dwivedi2020benchmarkgnns}, as it offers a broad choice of models, datasets and graph classification tasks.

\paragraph{\boldmath Selected $\GNNs$.}
We select the best architectures from  \citet{dwivedi2020benchmarkgnns}: Graph Attention Networks ($\mathsf{GAT}$) \citep{GAT}, Graph Convolutional Networks ($\mathsf{GCN}$) \citep{kipf-loose}, $\mathsf{GraphSage}$ \citep{hyl17}, Gaussian Mixture Models ($\mathsf{MoNet}$) \citep{monet} and $\mathsf{GatedGCN}$ \citep{gated}. We leave out various linear architectures such as $\mathsf{GIN}$ \citep{xhlj19} as they were shown to perform poorly on the benchmark. 

\paragraph{Learning tasks and datasets.} As in \citet{dwivedi2020benchmarkgnns} we consider (i)~graph regression and the ZINC dataset \citep{ZINCoriginal,dwivedi2020benchmarkgnns}; (ii)~vertex classification and the PATTERN and CLUSTER datasets \citep{dwivedi2020benchmarkgnns}; and (iii)~link prediction and the COLLAB dataset  \citep{hu2020open}. We omit graph classification: for this task, the graph datasets from \citet{dwivedi2020benchmarkgnns} originate from image data and hence vertex neighborhoods carry little information. 

\paragraph{Patterns.} We extend the initial features of vertices with 
homomorphism counts of cycles $C_\ell$ of length $\ell\leq 10$, when molecular data (ZINC) is concerned, and with homomorphism counts of $k$-cliques $K_k$ for $k\leq 5$, when social or collaboration data  (PATTERN, CLUSTER, COLLAB) is concerned. We use the $z$-score of the logarithms of homomorphism counts to make them standard-normally distributed and comparable to other features.
Section~\ref{sec:patterns} tells us that all these patterns will increase expressive power (Theorem~\ref{thm:increase} and Proposition~\ref{prop:cycles}) and are ``minimal''
in the sense that they are not the join of smaller patterns (Proposition~\ref{prop:simplify}). Similar pattern choices were used in
\citet{bouritsas2020improving}. We use DISC \citep{ZhangY0ZC20}\footnote{We thank the authors for providing us with an executable.},
 a tool specifically built to get homomorphism counts for large datasets. Each model is trained and tested independently using combinations of  patterns.\looseness=-1

\paragraph{\boldmath Higher-order $\GNNs$.} 
We do not compare to higher-order $\GNNs$ since this was already done by \citet{dwivedi2020benchmarkgnns}. They included ring-$\GNNs$ (which outperform $\mathsf{2WL}$-$\mathsf{GNNs}$) and $\mathsf{3WL}$-$\mathsf{GNNs}$ in their experiments, and these were outperformed by our selected ``linear'' architectures. Although the increased expressive 
power of higher-order $\GNNs$ may be beneficial for learning, scalability and learning issues (e.g., loss divergence) hamper their applicability \citep{dwivedi2020benchmarkgnns}.  Our approach thus certainly outperforms higher-order $\GNNs$ with respect to the benchmark. 

\paragraph{Methodology.}
Graphs were divided between training/test as instructed by \citet{dwivedi2020benchmarkgnns}, and all numbers reported correspond to the test set. The reported performance is the average over four runs with different random seeds for the respective combinations of patterns in $\mathcal{F}$, model and dataset. Training times were comparable to the baseline of training models without any augmented features.\footnote{Code to reproduce our experiments is available at \url{https://github.com/LGP-GNN-2021/LGP-GNN}} All models for ZINC, PATTERN and COLLAB were trained on a GeForce GTX 1080 Ti GPU, for CLUSTER a Tesla V100-SXM3-32GB GPU was used.

\smallskip
\noindent
Next we summarize our results for each learning task separately. 

\paragraph{Graph regression.}
The first task of the benchmark is the prediction of the solubility of molecules in the ZINC dataset \citep{ZINCoriginal,dwivedi2020benchmarkgnns}, a dataset of about $12\,000$ graphs of small size, each of them consisting of 
one particular molecule. 
The results in  Table \ref{ZINC_table} show that each of our models indeed improves by adding homomorphism counts of cycles and the best result
is obtained by considering all cycles.  $\GSNs$  were applied to the ZINC dataset as well  \citep{bouritsas2020improving}.
In Table \ref{ZINC_table} we also report
results by using subgraph isomorphism counts (as in $\GSNs$): homomorphism counts generally provide better results than subgraph isomorphisms counts. Our best result (framed in  Table \ref{ZINC_table}) is competitive to the value of $0.139$ reported in \citet{bouritsas2020improving}. By looking at the full results, we see that some cycles are much more important than others. Table \ref{GAT_table} shows which cycles have greatest impact for the worst-performing baseline, $\mathsf{GAT}$.
Remarkably, adding homomorphism counts makes the $\mathsf{GAT}$ model competitive with the best performers of the benchmark. 

\begin{table}[t]
\caption{Results for the ZINC dataset show that homomorphism (hom) counts of cycles improve every model. We compare the mean absolute error (MAE) of each model without any homomorphism count (baseline), against the model augmented with the  hom count,  and with subgraph isomorphism (iso) counts of $C_3$--$C_{10}$.}
	\label{ZINC_table}
	\vskip 0.15in
\centering
			\begin{sc}
				\begin{tabular}{p{2.5cm}p{2.75cm}p{2.7cm}p{2.7cm}}	
					\toprule
					Model  & MAE (base) & MAE (hom) & MAE (iso)\\
					\midrule
					$\mathsf{GAT}$   & 0.47$\pm$0.02 & \textbf{0.22}$\pm$\textbf{0.01} & 0.24$\pm$0.01\\
					$\mathsf{GCN}$    & 0.35$\pm$0.01 & \textbf{0.20}$\pm$\textbf{0.01} & 0.22$\pm$0.01\\
					$\mathsf{GraphSage}$  & 0.44$\pm$0.01 & \textbf{0.24}$\pm$\textbf{0.01} & 0.24$\pm$0.01\\
					$\mathsf{MoNet}$   & 0.25$\pm$0.01 & 0.19$\pm$0.01 & \textbf{0.16}$\pm$\textbf{0.01}\\
					$\mathsf{GatedGCN}$   & 0.34$\pm$0.05 & \!\!\!\fbox{\textbf{0.1353}$\pm$\textbf{0.01}} & 0.1357$\pm$0.01\\
					\bottomrule
				\end{tabular}
			\end{sc}
\end{table}
\begin{table}[t]
			\caption{The effect of different cycles for the $\mathsf{GAT}$ model over the ZINC dataset, using mean absolute error. }
		\label{GAT_table}
		\vskip 0.15in
		\centering
				\begin{sc}
					\begin{tabular}{p{2.5cm}c}
						\toprule
						Set $(\mathcal F)$ &  MAE \\
						\midrule
						None& 0.47$\pm$0.02  \\
						$\{C_3\}$ & 0.45$\pm$0.01 \\
						$\{C_4\}$ & 0.34$\pm$0.02 \\
						$\{C_6\}$ & 0.31$\pm$0.01 \\
		$\{C_5,C_6\}$ & 0.28$\pm$0.01 \\
						$\{C_3,\ldots,C_6\}$ & 0.23$\pm$0.01\\
						$\{C_3,\ldots,C_{10}\}$ & \textbf{0.22}$\pm$\textbf{0.01}\\
												\bottomrule
					\end{tabular}
				\end{sc}
\end{table}

\paragraph{Vertex classification.}
The next task in the benchmark corresponds to vertex classification. Here we analyze two datasets, PATTERN and CLUSTER \citep{dwivedi2020benchmarkgnns}, both containing over $12\,000$ artificially generated graphs 
 resembling social networks or communities. The task is to predict whether a vertex belongs to a particular cluster or pattern, and 
all results are measured using the accuracy of the classifier. 
Also here, our results show that homomorphism counts, this times of cliques, tend to improve the accuracy of our models. Indeed, for the PATTERN dataset we see an improvement in all models but $\mathsf{GatedGCN}$ 
(Table \ref{PATTERN_table}), and three models are improved in the CLUSTER dataset (reported in the appendix). Once again, the best performer in this task is a model that uses homomorphism counts. We remark that for cliques, homomorphism counts coincide with subgraph isomorphism counts (up to a constant factor) so our extensions behave like $\GSNs$.

\begin{table}[t]
	\caption{Results for the PATTERN dataset show that homomorphism counts improve all models except $\mathsf{GatedGCN}$. We compare weighted accuracy of each model without any homomorphism count (baseline) against the model augmented with the counts of the set $\mathcal F$ that showed best performance (best $\mathcal F$).}
	\label{PATTERN_table}
	\vskip 0.15in
	\centering
	\begin{sc}
				\begin{tabular}{p{5.2cm}p{4.9cm}p{4.9cm}}
					\toprule
					Model + best $\mathcal F$ &  Accuracy baseline & Accuracy best\\
					\midrule
					$\mathsf{GAT}$\!  $\{K_3,K_4,K_5\}$    & 78.83 $\pm$ 0.60  & 85.50 $\pm$ 0.23  \\
					$\mathsf{GCN}$\!  $\{K_3,K_4,K_5\}$    &  71.42 $\pm$ 1,38 &  82.49 $\pm$ 0.48\\
					$\mathsf{GraphSage}$ $\{K_3,K_4,K_5\}$ & 70.78 $\pm$ 0,19 & 85,85 $\pm$ 0.15 \\
					$\mathsf{MoNet}$  $\{K_3,K_4,K_5\}$   & 85.90 $\pm$ 0,03 &  \textbf{86.63} $\pm$ \textbf{0.03}\\
					$\mathsf{GatedGCN}$ $\{\emptyset\}$ & 86.15 $\pm$ 0.08  & 86.15 $\pm$ 0.08 \\
					\bottomrule
				\end{tabular}
			\end{sc}
\end{table}

\vspace{-2ex}
\paragraph{Link prediction}
In our final task we consider a single graph, COLLAB \citep{hu2020open}, with over $235\,000$ vertices, containing information about the collaborators in an academic network, and the task at hand is to predict future collaboration. The metric used in the benchmark is the Hits@50 evaluator \citep{hu2020open}. Here, positive collaborations are ranked among randomly sampled negative collaborations, and the metric is the ratio of positive edges that are ranked at place 50 or above.
Once again, homomorphism counts of cliques improve the performance of all models, see Table \ref{COLLAB_table}. An interesting observation is that this time the best set of features (cliques) does depend on the model, although the best model uses all cliques again.

\begin{table}[t]
	\caption{All models improve the Hits@50 metric over the COLLAB dataset. We compare each model without any homomorphism count (baseline) against the model augmented with the counts of the set of patterns that showed best performance (best $\mathcal F$).}
	\label{COLLAB_table}
	\vskip 0.15in
	\centering
			\begin{sc}
	\begin{tabular}{p{5.2cm}p{4.9cm}p{4.9cm}}
					\toprule
					Model + best $\mathcal F$&  Hits@50 baseline & Hits@50 best \\
					\midrule
					$\mathsf{GAT}$  \!  $\{K_3\}$   & 50.32$\pm$0.55 & 52.87$\pm$0.87 \\
					$\mathsf{GCN}$ \!  $\{K_3,K_4,K_5\}$    & 51.35$\pm$1.30 & \textbf{54.60}$\pm$\textbf{1.01} \\
					$\mathsf{GraphSage}$ \!  $\{K_5\}$    & 50.33$\pm$0.68 & 51.39$\pm$1.23 \\
					$\mathsf{MoNet}$ \!  $\{K_4\}$    & 49.81$\pm$1.56 & 51.76$\pm$1.38 \\
					$\mathsf{GatedGCN}$ \!  $\{K_3\}$  & 51.00$\pm$2.54 & 51.57$\pm$0.68 \\
					\bottomrule
				\end{tabular}
			\end{sc}
\end{table}

\paragraph{Remarks.}
The best performers in each task use homomorphism counts,
in accordance with our theoretical results, showing that such counts do extend the power of $\MPNNs$.
Homomorphism counts are also cheap to compute. For COLLAB, the largest graph in our experiments, the homomorphism counts of all patterns we used, for all vertices, could be computed by  DISC \citep{ZhangY0ZC20} in less than 3 minutes. 
One important remark is that \textit{selecting} the best set of features is still a challenging endeavor. Our theoretical results help us streamline this search, but for now it is still an exploratory task. In our experiments we first looked at adding each pattern individually, and then tried with combinations of those that showed the best improvements. This feature selection strategy incurs considerable cost, both computational and environmental, and needs further investigation.

\section{Conclusion}
We propose local graph parameter enabled $\MPNNs$ as an efficient way to increase the expressive
power of $\MPNNs$. 
 The take-away message is that enriching features with
homomorphism counts of small patterns is a promising add-on to any $\GNN$ architecture, with little to no overhead.
Regarding future work, the problem of which parameters to choose deserves further study. 
In particular, we plan to provide a complete characterisation of when adding a new pattern to $\mathcal F$ adds expressive power to the $\wlk{\mathcal F}$-test.

\appendix

\section*{Appendix}
\section{Proofs of Section~\ref{sec:exp-power}}
We use the following notions. Let $G$ and $H$ be graphs, $v\in V_G$, $w\in V_H$, and $d\geq 0$. The vertices $v$ and $w$ are said to be indistinguishably by $\wlk{\mathcal F}$ in round $d$, denoted by $(G,v)\equiv_{\wlk{\mathcal{F}}}^{(d)} (H,w)$, iff
$\chi_{\mathcal F,G,v}^{(d)}=\chi_{\mathcal F,H,w}^{(d)}$. Similarly, 
$G$ and $H$ are said to be indistinguishable by $\wlk{\mathcal F}$ in round $d$, denoted by $G\equiv_{\wlk{\mathcal{F}}}^{(d)} H$, iff
$\ldbl \chi_{\mathcal F,G,v}^{(d)}\mid v\in V_G\rdbl=\ldbl\chi_{\mathcal F,H,w}^{(d)}\mid w\in V_H\rdbl$. Along the same lines, $v$ and $w$ are indistinguishable by an $\mathcal F$-$\MPNN$ $M$, denoted by  $(G,v)\equiv_{M,\mathcal F}^{(d)} (H,w)$, iff
$\mathbf{x}_{M,\mathcal F,G,v}^{(d)}=\mathbf{x}_{M, \mathcal F,H,w}^{(d)}$. Similarly, 
$G$ and $H$ are said to be indistinguishable by $M$ in round $d$, denoted by $G\equiv_{M,\mathcal F}^{(d)} H$, iff
$\ldbl \mathbf{x}_{M,\mathcal F,G,v}^{(d)}\mid v\in V_G\rdbl=\ldbl\mathbf{x}_{M,\mathcal F,H,w}^{(d)}\mid w\in V_H\rdbl$.

\subsection{Proof of Proposition \ref{prop:wlupper}}
We show that the class of $\mathcal F$-$\MPNNs$ is upper bounded in expressive power by $\wlk{\mathcal F}$. The proof is analogous to the proof in \citet{grohewl} showing that $\MPNNs$ are bounded by $\wl{}$.

We show a stronger result by upper bounding $\mathcal F$-$\MPNNs$ by $\wlk{\mathcal F}$-test, layer by layer.
More precisely, we show that for every $\mathcal F$-$\MPNN$ $M$, graphs $G$ and $H$,
vertices $v\in V_G$, $w\in V_H$, and $d\geq 0$,
\begin{itemize}\setlength{\itemsep}{0pt}
  \setlength{\parskip}{0pt}
\item[\textbf{(1)}] $(G,v)\equiv_{\wlk{\mathcal{F}}}^{(d)} (H,w) \ \ \Longrightarrow \ \ (G,v)\equiv_{M,\mathcal{F}}^{(d)} (H,w)$; and
\item[\textbf{(2)}] $G\equiv_{\wlk{\mathcal{F}}}^{(d)}H \ \ \Longrightarrow \ \ G\equiv_{M,\mathcal{F}}^{(d)} H$.
\end{itemize}
Clearly, these imply that $\mathcal F$-$\MPNNs$ are bounded in expressive power by $\wlk{\mathcal F}$, both when vertex and graph distinguishability are concerned.

\smallskip
\noindent
\textbf{Proof of implication (1).} We show this implication by induction on the number of rounds.

\noindent
\underline{Base case.}
We first assume $(G, v) \equiv_{\wlk{\mathcal F}}^{(0)} (H,w)$. In other words,
$\chi_{\mathcal F,G,v}^{(0)}=\chi_{\mathcal F,H,w}^{(0)}$ and thus,
$\chi_{G}(v)=\chi_H(w)$ and for every $P^r\in\mathcal F$ we have 
$
\homc{P^r,G^v}=\homc{P^r,H^w}
$.
By definition, $\mathbf{x}^{(0)}_{M,\mathcal F,G,v}$ is a hot-one encoding of $\chi_G(v)$ combined with $\homc{P^r,G^v}$ for $P^r\in\mathcal F$, for every $\MPNN$ $M$, graph $G$ and vertex $v \in V_G$. Since these agree with the labelling and homomorphism counts for vertex $w\in V_H$ in graph $H$, we also have that 
$\mathbf{x}^{(0)}_{M,\mathcal F,G,v}=\mathbf{x}^{(0)}_{M,\mathcal F,H,w}$, as desired.

\noindent
\underline{Inductive step.}
We next assume $(G, v) \equiv_{\wlk{\mathcal F}}^{(d)} (H,w)$. By the definition of $\wlk{\mathcal F}$ this is equivalent to 
$(G, v) \equiv_{\wlk{\mathcal F}}^{(d-1)}  (H,w)$ and 
$\ldbl \chi_{\mathcal F, G,v'}^{(d-1)}\mid v'\in N_G(v)\rdbl=\ldbl \chi_{\mathcal F, H,w'}^{(d-1)}\mid w'\in N_H(w)\rdbl$. By the induction hypothesis, this implies 
$(G, v) \equiv_{M,\mathcal F}^{(d-1)}  (H,w)$ and there exists a bijection $\beta:N_G(v)\to N_H(w)$
such that $(G, v') \equiv_{M,\mathcal F}^{(d-1)}  (H,\beta(v'))$ for every $v'\in N_G(v)$, and every $\mathcal F$-$\MPNN$ $M$. In other words,
$\mathbf{x}^{(d-1)}_{M,\mathcal F,G,v}=\mathbf{x}^{(d-1)}_{M,\mathcal F,H,w}$ and 
$\mathbf{x}^{(d-1)}_{M,\mathcal F,G,v'}=\mathbf{x}^{(d-1)}_{M,\mathcal F,H,\beta(v')}$ for every $v'\in N_G(v)$. By the definition of $\mathcal F$-$\MPNNs$ this implies that 
$\textsc{Comb}^{(d)}\bigl(\ldbl \mathbf{x}^{(d-1)}_{M,\mathcal F,G,v'} \mid v'\in N_G(v)\rdbl\bigr)$ is equal to 
$\textsc{Comb}^{(d)}\bigl(\ldbl \mathbf{x}^{(d-1)}_{M,\mathcal F,H,w'} \mid w'\in N_G(w)\rdbl\bigr)$ and hence also, after applying $\textsc{Upd}^{(d)}$,
$\mathbf{x}^{(d)}_{M,\mathcal F,G,v}=\mathbf{x}^{(d)}_{M,\mathcal F,H,w}$. That is,
$(G, u) \equiv_{M,\mathcal F}^{(d)}  (H,w)$, as desired.

\smallskip
\noindent
\textbf{Proof of implication (2).}
The implication $G \equiv_{\wlk{\mathcal{F}}}^{(d)} H \ \ \Longrightarrow \ \ G \equiv_{M,\mathcal{F}}^{(d)} H$ now easily follows. Indeed, $G \equiv_{\wlk{\mathcal{F}}}^{(d)} H$ is equivalent to $\ldbl \chi_{\mathcal F,G,v}^{(d)}\mid v\in V_G\rdbl=\ldbl \chi_{\mathcal F,H,w}^{(d)}\mid w\in V_H\rdbl$. In other words, there exists a bijection $\alpha:V_G\to V_H$ such that 
$\chi_{\mathcal F,G,v}^{(d)}=\chi_{\mathcal F,H,\alpha(v)}^{(d)}$ for every $v\in V_G$. 
 We have just shown that this implies  $\mathbf{x}^{(d)}_{M,\mathcal F,G,v}=\mathbf{x}^{(d)}_{M,\mathcal F,H,\alpha(v)}$ for every $v\in V_G$ and
for every $\mathcal F$-$\MPNN$ $M$. Hence, $\ldbl \mathbf{x}^{(d)}_{M,\mathcal F,G,v}\mid v\in V_G\rdbl=\ldbl \mathbf{x}^{(d)}_{M,\mathcal F,H,w}\mid w\in V_H\rdbl$, or $G\equiv_{M,\mathcal F}^{(d)} H$, as desired. \qed

\subsection{Proof of Theorem~\ref{thm:char}}
We show that for any finite collection $\mathcal{F}$ of patterns, graphs $G$ and $H$, vertices $v\in V_G$ and $w\in V_H$, and $d\geq 0$:
\begin{equation}
(G,v)\equiv_{\wlk{\mathcal{F}}}^{(d)} (H,w) \ \Longleftrightarrow   \  \mathsf{hom}(T^r,G^v)=\mathsf{hom}(T^r,H^w), \label{implication-vertex}
\end{equation}
for every $\mathcal{F}$-pattern tree $T^r$ of depth at most $d$. Similarly, 
\begin{equation}
G\equiv_{\wlk{\mathcal{F}}}^{(d)} H \ \Longleftrightarrow   \  \mathsf{hom}(T,G)=\mathsf{hom}(T,H), \label{implication-graph}
\end{equation}
for every (unrooted) $\mathcal{F}$-pattern tree of depth at most $d$.

For a given set $\mathcal F=\{P_1^r,\ldots, P_\ell^r\}$ of patterns and $\mathbf{s}=(s_1,\ldots,s_\ell)\in\Nb^{\ell}$, we denote by 
$\mathcal F^{\mathbf{s}}$ the graph pattern of the form 
$
(P_1^{s_1}\star\cdots\star P_\ell^{s_\ell})^r
$, that is, we join $s_1$ copies of $P_1$, $s_2$ copies of $P_2$ and so on.

\smallskip
\noindent
\textbf{Proof of equivalence~(\ref{implication-vertex}).}
The proof is by induction on the number of rounds $d$.

\fbox{$\Longrightarrow$}
We first consider the implication 
$(G,v)\equiv_{\wlk{\mathcal F}}^{(d)} (H,w)
\Longrightarrow \homc{T^r,G^v}=\homc{T^r,H^w}$
for every $\mathcal F$-pattern tree $T^r$ of depth at most $d$.

\noindent
\underline{Base case.}
Let us first consider the base case, that is, $d=0$. In other words, we consider $\mathcal F$-pattern trees $T^r$ consisting of a single root $r$ adorned with a pattern $\mathcal F^{\mathsf{s}}$ for some $\mathbf{s}=(s_1,\ldots,s_\ell)\in\Nb^\ell$.
We note that due to the properties of the graph join operator:
\begin{equation}
\homc{T^r,G^v}=\prod_{i=1}^\ell \bigl(\homc{P_i^{r},G^v} \bigr)^{s_i}. \label{eq:base}
\end{equation}
Since,  $(G,v)\equiv_{\wlk{\mathcal F}}^{(0)} (H,w)$, we know that 
$\chi_G(v)=\chi_H(w)=a$ for some $a\in\Sigma$ and 
$\homc{P_i^r,G^v}=\homc{P_i^r,H^w}$ for all
$P_i^r\in\mathcal F$. This implies that the product in~(\ref{eq:base}) is equal to 
$$
\prod_{i=1}^\ell \bigl(\homc{P_i^{r},H^w}\bigr)^{s_i}=
\homc{T^r,H^w},
$$
as desired.

\noindent
\underline{Inductive step.}
Suppose next that we know that the implication holds for $d-1$. We assume now $(G,v)\equiv_{\wlk{\mathcal F}}^{(d)} (H,w)$ and consider an $\mathcal F$-pattern tree $T^r$ of depth at most $d$. Assume that in the backbone of $T^r$, the root $r$ has $m$ children $c_1,\ldots,c_m$, and denote by $T_1^{c_1},\ldots, T_m^{c_\ell}$ the 
$\mathcal F$-pattern trees in $T^r$ rooted at $c_i$.
Furthermore, we denote by $T_i^{(r,c_i)}$ the $\mathcal F$-pattern tree obtained from $T_i^{c_i}$ by attaching $r$ to $c_i$; $T_i^{(r,c_i)}$ has root $r$.
Let $\mathcal F^{\mathbf{s}}$ be the pattern in $T^r$ associated with $r$. The following equalities are readily verified:
\begin{align}
\homc{T^r,G^v}&=\homc{\mathcal F^{\mathbf{s}},G^v}\prod_{i=1}^m \homc{T_i^{(r,c_i)},G^v}
=\homc{\mathcal F^{\mathbf{s}},G^v}\prod_{i=1}^m \Bigl(\sum_{v'\in N_G(v)}
\homc{T_i^{c_i},G^{v'}}\Bigr).\label{eq:induct}
\end{align}
Recall now that we assume  $(G,v)\equiv_{\wlk{\mathcal F}}^{(d)} (H,w)$ and thus, in particular,
 $(G,v)\equiv_{\wlk{\mathcal F}}^{(0)} (H,w)$.
 Hence, by induction, $\homc{S^r,G^v}=\homc{S^r,H^w}$ for every 
 $\mathcal F$-pattern tree $S^r$ of depth $0$. In particular, this holds for $S^r=\mathcal F^{\mathbf{s}}$ and hence
 $$
 \homc{\mathcal F^{\mathbf{s}},G^v}=
  \homc{\mathcal F^{\mathbf{s}},H^w}.
 $$
 Furthermore, $(G,v)\equiv_{\wlk{\mathcal F}}^{(d)} (H,w)$ implies that 
  there exists a  bijection $\beta:N_G(v)\to N_H(w)$ 
  such that $(G,v')\equiv_{\wlk{\mathcal F}}^{(d-1)} (H,\beta(v'))$ for every $v'\in N_G(v)$. By induction, for every $v'\in N_G(v)$ there thus exists a unique $w'\in N_H(w)$ such that 
 $\homc{S^r,G^{v'}}=\homc{S^r,H^{w'}}$
 for every $\mathcal F$-pattern tree $S^r$ of depth at most $d-1$. In particular, 
 for every $v'\in N_G(v)$ there exists a $w'\in N_H(w)$ such that 
$$\homc{T_i^{c_i},G^{v'}}=\homc{T_i^{c_i},H^{w'}},$$
for each of the sub-trees $T_i^{c_i}$ in $T^r$. Hence,~(\ref{eq:induct}) is equal to 
$$\homc{\mathcal F^{\mathbf{s}},H^w}\prod_{i=1}^m \Bigl(\sum_{w'\in N_H(w)}\homc{T_i^{c_i},H^{w'}}\Bigr),
$$
which in turn is equal to $\homc{T^r,H^w}$, as desired.

\fbox{$\Longleftarrow$}
We next consider the other direction, that is, we show that when $\homc{T^r,G^v}
=\homc{S^r,H^w}$ holds for every $\mathcal F$-pattern tree $T^r$ of depth at most $d$, then $(G,v)\equiv_{\wlk{\mathcal F}}^{(d)} (H,w)$ holds. This is again verified by induction on $d$. This direction is more complicated and is similar to techniques used in \citet{Grohe20Lics}. In our induction hypothesis we further include that a \textit{finite} number of $\mathcal F$-pattern trees suffices to infer $(G,v)\equiv_{\wlk{\mathcal F}}^{(d)} (H,w)$ for graphs $G$ and $H$ and vertices $v\in V_G$ and $w\in V_H$.

\smallskip
\underline{Base case.}
Let us consider the base case $d=0$ first.
We need to show that $\chi_G(v)=\chi_H(w)$ and $\homc{P_i^r,G^v}=\homc{P_i^r,H^w}$ for
every $P_i^r\in\mathcal F$, since this implies $(G,v)\equiv_{\wlk{\mathcal F}}^{(0)} (H,w)$.

 We first observe that $\homc{T^r,G^v}=\homc{T^r,H^w}$ for every $\mathcal F$-pattern tree $T^r$ of depth $0$, implies that $v$ and $w$ must be assigned the same label, say $a$, by $\chi_G$ and $\chi_H$, respectively.

Indeed, if we take $T^r$ to consist of a single root $r$ labeled with $a$ (and thus
$r$ is associated with the pattern $\mathcal F^{\mathbf{0}}$),
then 
$\homc{T^r,G^v}=\homc{T^r,H^w}$ will be one if $\chi_G(v)=\chi_H(w)=a$ and zero otherwise.
This implies that $\chi_G(v)=\chi_H(w)=a$.

Next, we show that $\homc{P_i^r,G^v}=\homc{P_i^r,H^w}$ for every $P_i^r\in\mathcal F$.
It suffices to consider the $\mathcal F$-pattern tree $T^r_i$ consisting of a root $r$ joined with a single copy of $P_i^r$.

We observe that we only need a finite number of $\mathcal F$-pattern trees to infer  $(G,v)\equiv_{\wlk{\mathcal F}}^{(0)} (H,w)$. Indeed, suppose that $\chi_G$ and $\chi_H$ assign labels $a_1,\ldots,a_L$,
then we need $L$ single vertex trees with no patterns attached and root labeled with one of these labels. In addition, we need one $\mathcal F$-pattern tree for each pattern $P_i^r\in \mathcal F$ and each label $a_1,\ldots,a_L$. That is, we need $L(\ell+1)$ $\mathcal F$-pattern trees of depth $0$.

\smallskip
\underline{Inductive step.}
We now assume that the implication holds for $d-1$ and consider $d$. That is, we assume that if $\homc{T^r,G^v}=\homc{T^r,H^w}$ holds for every $\mathcal F$-pattern tree $T^r$ of depth at most $d-1$, then $(G,v)\equiv_{\wlk{\mathcal F}}^{(d-1)} (H,w)$ holds. Furthermore,
we assume that only a finite number $K$ of $\mathcal F$-pattern trees $S_1^r,\ldots,S_K^r$ of depth at most $d-1$ suffice to infer $(G,v)\equiv_{\wlk{\mathcal F}}^{(d-1)} (H,w)$.

So, for $d$, let us assume that
$\homc{T^r,G^v}=\homc{T^r,H^w}$ holds for every $\mathcal F$-pattern tree of depth at most $d$. We need to show  $(G,v)\equiv_{\wlk{\mathcal F}}^{(d)}(H,w)$ and that we can again assume that a finite number of $\mathcal F$-pattern trees of depth at most $d$ suffice to infer $(G,v)\equiv_{\wlk{\mathcal F}}^{(d)}(H,w)$.

By definition of $(G,v)\equiv_{\wlk{\mathcal F}}^{(d)}(H,w)$, we can, equivalently, show that
$(G,v)\equiv_{\wlk{\mathcal F}}^{(d-1)}(H,w)$
and that there exists  a bijection $\beta:N_G(v)\to N_H(w)$ such that 
$(G,v')\equiv_{\wlk{\mathcal F}}^{(d-1)}(H,\beta(v'))$
for every $v'\in N_G(v)$. 
That $(G,v)\equiv_{\wlk{\mathcal F}}^{(d-1)}(H,w)$ holds, is by induction, since $\homc{T^r,G^v}=\homc{T^r,H^w}$ for every $\mathcal F$-pattern tree of depth at most $d$ and thus also for every $\mathcal F$-pattern tree of depth at most $d-1$. We may thus focus on showing the existence of the bijection $\beta$.

Let $X,Y\in\{G,H\}$, $x\in V_X$ and $y\in V_Y$. We know, by induction and
the proof of the previous implication, 
 that $(X,x)\equiv_{\wlk{\mathcal F}}^{(d-1)} (Y,y)$ if and only if $\homc{S_i^r,X^x
}=\homc{S_i^r,Y^y}$ for each $i\in K$. 
Denote by $R_1,\ldots,R_e$  the equivalence class on
$V_X\cup V_Y$ induced by $\equiv_{\wlk{\mathcal F}}^{(d-1)}$.
Furthermore, define $N_{j,X}(x):=N_X(x)\cap R_j$ and let  $n_j=|N_{j,G}(v)|$ and $m_j=|N_{j,H}(w)|$ for $v\in V_G$ and $w\in V_H$, for  each $j\in [e]$. If we can show that $n_j=m_j$ for each $j\in[e]$, then this implies the existence of the desired bijection. 

Let $T_i^{r=a}$ be the $\mathcal F$-pattern tree of depth at most $d$ obtained by attaching $S_i^r$ to a new root vertex $r$ labeled with $a$.  We may assume that $v$ and $w$ both have label $a$, since their homomorphism counts for the single root trees with labels from $\Sigma$.
The root vertex $r$ is not joined with any $\mathcal F^{\mathbf{s}}$ (or alternatively it is joined with $\mathcal F^{\mathbf{0}}$). It will be convenient to denote the root of $S_i^r$ by $r_i$ instead of $r$. Then for each $i\in [K]$:
\allowdisplaybreaks
\begin{align*}
\homc{T_i^{r=a},G^v}&=\sum_{v'\in N_G(v)}\homc{S_i^{r_i},G^{v'}}=\sum_{j\in[e]} n_j\homc{S_i^{r_i},G^{v'_j}}\\
&=\sum_{j\in[e]} m_j\homc{S_i^{r_i},H^{w_j'}}=
\sum_{w'\in N_H(w)}\homc{S_i^{r_i},H^{w'}}=\homc{T_i^{r=a},H^w},
\end{align*}
where $v_j'$ and $w_j'$ denote arbitrary vertices in $N_{j,G}(v)$ and $N_{j,H}(w)$, respectively. Let us denote $\homc{S_i^{r_i},G^{v_j'}}$ by $a_{ij}$ and observe that this is equal to $\homc{S_i^{r_i},H^{w_j'}}$. Hence, we
know that for each $i\in[K]$: 
$$
\sum_{j\in[e]} a_{ij}n_j=\sum_{j\in[e]} a_{ij}m_j.
$$

Let us call a set $I\subseteq[K]$ compatible if all roots in $S_i^{r_i}$, for  $i\in I$, have the same label. Consider a vector $\mathbf{s}=(s_1,\ldots,s_K)\in\Nb^{K}$ and define its support as $\mathsf{supp}(\mathbf{s}):=\{ i\in [K]\mid s_i\neq 0\}$. We say that $\mathbf{s}$ is compatible if its support is. For such a compatible $\mathbf{s}$ we now define 
$T^{r=a,\mathbf{s}}$ to be the $\mathcal F$-pattern tree with root $r$ labeled with $a$, with one child $c$ which is joined with (and inheriting the label from) the following $\mathcal F$-pattern tree of depth $d-1$:
$$
\bigstar_{i\in\mathsf{supp}(\mathbf{s})} S_i^{s_i}.
$$
In other words, we simply join together powers of the $S_i^{r_i}$'s that have roots with the same label. Then for every compatible $\mathbf{s}\in\Nb^K$:
\allowdisplaybreaks
\begin{align*}
\homc{T^{r=a,\mathbf{s}},G^v}&=\sum_{v'\in N_G(v)} \prod_{i\in[K]}\bigl(\homc{S_i^{r_i},G^{v'}}\bigr)^{s_i}
=\sum_{j\in[e]} n_j \prod_{i\in[K]}\bigl(\homc{S_i^{r_i},G^{v_j'}}\bigr)^{s_i}\\
&=\sum_{j\in[e]} m_j \prod_{i\in[K]}\bigl(\homc{S_i^{r_i},H^{w_j'}}\bigr)^{s_i}
=\sum_{w'\in N_H(w)} \prod_{i\in[K]}\bigl(\homc{S_i^{r_i},H^{w'}}\bigr)^{s_i}\\
&=\homc{T_i^{r=a,\mathbf{s}},H^w},
\end{align*}
where, as before, $v_j'$ and $w_j'$ denote arbitrary vertices in $N_{j,G}(v)$ and $N_{j,H}(w)$, respectively. Hence, for any compatible  $\mathbf{s}\in\Nb^{K}$:
$$
\sum_{j\in[e]} n_j \prod_{i\in[K]} a_{ij}^{s_i}
=
\sum_{j\in[e]} m_j \prod_{i\in[K]} a_{ij}^{s_i}.
$$
We now continue in the same way as in the proof of Lemma 4.2 in \citet{Grohe20Lics}.
We repeat the argument here for completeness. Define $\mathbf{a}_j^{\mathbf{s}}:=\prod_{i=1}^K a_{ij}^{s_i}$ for each $j\in[e]$.
We assume, for the sake of contradiction, that there exists a $j\in[e]$ such that $n_j\neq m_j$. We choose such a $j_0\in[e]$ for which  $S=\mathsf{supp}(\mathbf{a}_{j_0})$ is inclusion-wise maximal.

We first rule out that $S=\emptyset$. Indeed, suppose that $S=\emptyset$. This implies that $\mathbf{a}_{j_0}=\mathbf{0}$. Now observe that 
$\mathbf{a}_j$ and $\mathbf{a}_{j'}$ are mutually distinct for all $j,j'\in[e]$, $j\neq j'$. Indeed, if they were equal then this would imply that $R_j=R_{j'}$.
Hence, $\mathsf{supp}(\mathbf{a}_j)\neq \emptyset$ for any $j\neq j_0$. We note that $n_j=m_j$ for all $j\neq j_0$ by the maximality of $S$. Hence, $n_{j_0}=n-\sum_{j\neq j_{0}}n_j=n-\sum_{j\neq {j_0}}m_j=m_{j_0}$, contradicting our assumption. Hence, $S\neq \emptyset$.

Consider $J:=\{ j\in[e]\mid \mathsf{supp}(\mathbf{a}_j)=S\}$. For each $j\in J$, consider the truncated vector $\hat{\mathbf{a}}_j:=(a_{ij}\mid i\in S)$.
We note that $\hat{\mathbf{a}}_j$, for $j\in J$, all have positive entries and are mutually distinct. Lemma 4.1 in \citet{Grohe20Lics} implies that we can find a vector (with non-zero entries) $\hat{\mathbf{s}}=(\hat{s}_i\mid i\in S)$ such that the numbers $\hat{\mathbf{a}}_j^{\hat{\mathbf{s}}}$
for $j\in J$ are mutually distinct as well. We next consider $\mathbf{s}=(s_1,\ldots,s_K)$ with $s_i=\hat{s}_i$ if $i\in S$ and $s_i=0$ otherwise. Then by definition of $\hat{\mathbf{s}}$, also $\mathbf{a}_j^{\mathbf{s}}$ for $j\in J$ are mutually distinct.

We next note that for every $p\in\Nb$,
$\mathbf{a}_j^{p\mathbf{s}}=(\mathbf{a}_j^{\mathbf{s}})^p$ and if we define $\mathbf{A}$ to be the $|J|\times |J|$-matrix such that 
$A_{jj'}:=\mathbf{a}_j^{j'\mathbf{s}}$ then this will be an invertible matrix (Vandermonde). We use this invertibility to show that $n_{j_0}=m_{j_0}$.

Let $\mathbf{n}_J:=(n_j\mid j\in J)$ and $\mathbf{m}_J=(m_j\mid j\in J)$. If we inspect the $j'$th entry of $\mathbf{n}_J\cdot \mathbf{A}$, then this is equal to
\begin{align*}
\sum_{j\in J} n_j\mathbf{a}_j^{j'\mathbf{s}}=\sum_{j\in[e]} n_j\mathbf{a}_j^{j'\mathbf{s}}-\sum_{\substack{j\in[e]\\S\not\subseteq \mathsf{supp}(\mathbf{a}_j)}} n_j\mathbf{a}_j^{j'\mathbf{s}}
-\sum_{\substack{j\in[e]\\S\subset \mathsf{supp}(\mathbf{a}_j)}} n_j\mathbf{a}_j^{j'\mathbf{s}}.
\end{align*}
We want to reduce the above expression to
$$
\sum_{j\in J} n_j\mathbf{a}_j^{j'\mathbf{s}}=\sum_{j\in[e]} n_j\mathbf{a}_j^{j'\mathbf{s}}-\sum_{\substack{j\in[e]\\S\subset \mathsf{supp}(\mathbf{a}_j)}} n_j\mathbf{a}_j^{j'\mathbf{s}}.
$$
To see that this holds, we verify that 
when $S\not\subseteq\mathsf{supp}(\mathbf{a}_j)$ then $\mathbf{a}_j^{j'\mathbf{s}}=0$. Indeed,
take an $\ell\in S$ such that $\ell\not\in\mathsf{supp}(\mathbf{a}_j)$.
Then, $\mathbf{a}_j^{j'\mathbf{s}}$ contains the factor $a_{\ell j}^{j's_{\ell}}=0^{s_{\ell}}$ with $s_{\ell}=\hat{s}_\ell\neq 0$. Hence, $\mathbf{a}_j^{j'\mathbf{s}}=0$.

Now, by the maximality of $S$, for all $j$ with $S\subset \mathsf{supp}(\mathbf{a}_j)$ we have $n_j=m_j$ and thus 
$$
\sum_{\substack{j\in[e]\\S\subset \mathsf{supp}(\mathbf{a}_j)}} n_j\mathbf{a}_j^{j'\mathbf{s}}=
\sum_{\substack{j\in[e]\\S\subset \mathsf{supp}(\mathbf{a}_j)}} m_j\mathbf{a}_j^{j'\mathbf{s}}.
$$
Since $\sum_{j\in[e]} n_j\mathbf{a}_j^{j'\mathbf{s}}=\sum_{j\in[e]} m_j\mathbf{a}_j^{j'\mathbf{s}}$, we thus also have that 
$$
\sum_{j\in J} n_j\mathbf{a}_j^{j'\mathbf{s}}
=
\sum_{j\in J} m_j\mathbf{a}_j^{j'\mathbf{s}}.
$$
Since this holds for all $j'\in J$, we have 
$\mathbf{n}_J\cdot \mathbf{A}=\mathbf{m}_J\cdot \mathbf{A}$ and by the invertibility of $\mathbf{A}$,
$\mathbf{n}_J=\mathbf{m}_J$. In particular, since $j_0\in J$, $n_{j_0}=m_{j_0}$ contradicting our assumption.

As a consequence, $n_j=m_j$ for all $j\in [e]$ and thus we have our desired bijection.

It remains to verify that we only need a finite number of $\mathcal F$-pattern trees to conclude that $n_j=m_j$ for all $j\in [e]$. In fact, the above proof indicates that we just need to check test for each root label $a$, we need to check identities for the finite number of pattern trees used to define the matrix $\mathbf{A}$.

\noindent
\textbf{Proof of equivalence~\ref{implication-graph}}
This equivalence is shown just like proof of Theorem 4.4. in \citet{Grohe20} with the additional techniques from Lemma 4.2
in \citet{Grohe20Lics}.

\fbox{$\Longrightarrow$}
We first show that $G\equiv_{\wlk{\mathcal F}}^{(d)} H$ implies
$\homc{T,G}=\homc{T,H}$ for unrooted
$\mathcal F$-pattern trees $T$ of depth at most $d$.

Assume that $V_X\cap V_Y=\emptyset$ for $X,Y\in\{G,H\}$. For $x\in V_X$ and $y\in V_Y$, define $x\sim_d y$ if and only if $\homc{T^r,X^x}=\homc{T^r,Y^y}$ for
all $\mathcal F$-pattern trees $T^r$ of depth at most $d$.
Let $R_1,\ldots,R_e$ be the $\sim_d$-equivalence classes
and for each $j\in[e]$, let  $p_j:=|R_j\cap V_G|$ and $q_j:=|R_j\cap V_H|$. Suppose that 
$G\equiv_{\wlk{\mathcal F}}^{(d)} H$. This implies that 
$p_j=q_j$ for every $j\in[e]$. 

Let $T$ be an unrooted $\mathcal F$-pattern tree of depth at most $d$, let $r$ be any vertex on the backbone of $T$, and let $T^r$ be the rooted $\mathcal F$-pattern tree obtained from $T$ by declaring $r$ as its root.
By definition,
for $X\in\{G,H\}$, any $x\in V_X\cap R_j$, $\homc{T^r,X^x}$ are all the same number, only dependent on $j\in[e]$. Hence,
\allowdisplaybreaks
\begin{align*}
\homc{T,G}&=\sum_{v\in V(G)}\homc{T^r,G^v}=\sum_{j\in[e]} p_j \homc{T^r,G^{v_j}}\\
&=\sum_{j\in[e]} q_j \homc{T^r,H^{w_j}}=\sum_{w\in V(H)} \homc{T^r,H^w}=\homc{T,H},
\end{align*}
where $v_j$ and $w_j$ are arbitrary vertices in $R_j\cap V_G$ and $R_j\cap V_H$, respectively, and where we used that $\homc{T^r,G^{v_j}}=\homc{T^r,H^{w_j}}$ and $p_j=q_j$.
Since this holds for any unrooted $\mathcal F$-pattern tree $T$ of depth at most $d$, we have show the desired implication.

\fbox{$\Longleftarrow$} We next check the other direction. That is, we assume that
$\homc{T,G}=\homc{T,H}$ holds for any unrooted $\mathcal F$-pattern tree $T$ of depth at most $d$ and verify that $G\equiv_{\wlk{\mathcal F}}^{(d)} H$.

For $x\sim_d y$ to hold for $x\in V_X$, $y\in V_Y$ and $X,Y\in\{G,H\}$, we earlier showed that this corresponds to checking whether 
$\homc{T_i^{r_i},X^x}=\homc{T_i^{r_i},Y^y}$ for a \textit{finite} number $K$  rooted $\mathcal F$-pattern trees $T_i^{r_i}$. By definition of the $R_j$'s, 
$a_{ij}:=\homc{T_i^{r_i},X^x}$ for $x\in R_j$
is well-defined (independent of the choice of $X\in\{G,H\}$
$x\in V_X$). For the rooted $T_i^{r_i}$'s we denote by $T_i$ its unrooted version. Similarly as before,
$$
\homc{T_i,G}=\sum_{j\in[e]} a_{ij}p_j=
\sum_{j\in[e]} a_{ij}q_j=\homc{T_i,H}.
$$
We next show that $p_j=q_j$ for $j\in [e]$. In fact, this is shown in precisely the same way as in our previous characterisation and based on Lemma 4.2 in \citet{Grohe20Lics}. That is, we again consider 
trees obtained by joining copies of the $T_i$'s, to obtain, for compatible $\mathbf{s}\in\Nb^K$,
$$\sum_{j\in[e]} a_{ij}^{s_i}p_j=
\sum_{j\in[e]} a_{ij}^{s_i}q_j.
$$
It now suffices to repeat the same argument as before (details omitted). \qed

\section{Proofs of Section~\ref{sec:choice}}

\subsection{Additional details of standard concepts}

\paragraph{Core and treewidth.}
A graph $G$ is a {\em core} if all homomorphisms from $G$ to itself are injective. The {\em treewidth} of a graph $G = (V,E,\chi)$ is a measure of how much $G$
resembles a tree. This is defined in terms of the {\em tree decompositions} of $G$, which are pairs $(T,\lambda)$, for  a tree $T=(V_T,E_T)$ and $\lambda$ a mapping that
associates each vertex $t$ of $V_T$ with a set $\lambda(t) \subseteq V$, satisfying the following:
\begin{itemize}\setlength{\itemsep}{0pt}
  \setlength{\parskip}{0pt}
\item The union of $\lambda(t)$, for $t\in V_T$, is equal to $V$;
\item The set $\{t \in V_T \mid v \in \lambda(t)\}$ is connected, for all $v \in V$; and
\item For each $\{u,v\} \in E$ there is $t \in V_T$ with $\{u,v\} \in \lambda(t)$.
\end{itemize}
The {\em width} of $(T,\lambda)$ is $\min_{t\in T}(|\lambda(t)|) - 1$.
The treewidth of $G$ is the minimum width of its tree decompositions.
For instance, trees have treewidth one, cycles have clique two, and the $k$-clique $K_k$ has treewidth $k-1$ (for $k > 1$).

If $P^r$ is a pattern, then its treewidth is defined as the treewidth of the graph $P$.
Similarly, $P^r$ is a core if $P$ is.
\newcommand{\itp}{\mathsf{isotp}}

\paragraph{\boldmath $\wlk{k}$.}
A \emph{partial isomorphism} from a graph $G$ to a graph $H$ is
a set $\pi\subseteq
V_G\times V_H$ such that all $(v,w),(v',w')\in\pi$ satisfy the equivalences
$v=v'\Leftrightarrow w=w'$, $\{v,v'\}\in E_G\Leftrightarrow \{w,w'\}\in E_H$,
$\chi_G(v)=\chi_H(w)$ and $\chi_G(v')=\chi_H(w')$.
 We may view $\pi$ as a
 bijective mapping from a subset $X\subseteq V_G$ to a subset of
$Y\subseteq V_H$ that is an isomorphism from the induced subgraph $G[X]$
to the induced subgraph $H[Y]$.
The \emph{isomorphism type} $\itp(G,\bar v)$ of a $k$-tuple
$\bar v=(v_1,\ldots,v_k)$ is a label in some alphabet $\Sigma$
such that 
$\itp(G,\bar v)=\itp(H,\bar w)$ if and only if
 $\pi=\{(v_1,w_1),\ldots,(v_k,w_k)\}$ is a partial isomorphism from $G$
 to $H$.

Let $k\ge 1$ and $G=(V,E,\chi)$.
The $k$-dimensional Weisfeiler-Leman algorithm ($\wlk{k}$) computes a
sequence of labellings $\chi_{k,G}^{(d)}$ from $V^k\to\Sigma$. We denote
by $\chi_{k,G,\bar v}^{(d)}$ the label assigned to the $k$-tuple $\bar v\in V^k$ in round $d$.
The initial labelling $\chi_{k,G}^{(0)}$ assigns to each $k$-tuple
$\bar v$ is isomorphism type $\itp(G,\bar v)$. Then, for round $d$,
$$ \chi_{k,G,\bar v}^{(d)}:=\big(\chi_{k,G,\bar v}^{(d-1)},M^{(d-1)}_{\bar v}\big),
$$
where $M^{(d-1)}_{\bar v}$ is the multiset
$$
\Big\{\!\!\!\Big\{\big(\itp(v_1,\ldots,v_k,w),\chi_{k,G,(v_1,\ldots,v_{k-1},w)}^{(d-1)},\chi_{k,G,(v_1,\ldots,v_{k-2},w,v_k)}^{(d-1)},\ldots,\chi_{k,G,(w,v_2,\ldots,v_k)}^{(d-1)}\big)\:\Big|\;w\in V
 \Big\}\!\!\!\Big\}.
$$
As observed in \citet{DellGR18}, if $k\ge 2$ holds, then we can omit the entry
 $\itp(v_1,\ldots,v_k,w)$ from the tuples in~$M_{\bar v}$, because all
the information it contains is also contained in the entries $\chi_{k,G,\ldots}^{(d-1)}$ of these tuples. Also, $\wl{}=\wlk{1}$ in the sense that $\chi_{G,v}^{(d)}=\chi_{G,v'}^{(d)}$
if and only if $\chi_{1,G,v}^{(d)}=\chi_{1,G,v'}^{(d)}$ for all $v,v'\in V$.
The $\wlk{k}$ algorithm is run until the labelings stabilises, i.e., 
 if for all $\bar v,\bar w\in V^k$, $\chi_{k,G,\bar v}^{(d)}=\chi_{k,G,\bar w}^{(d)}$
 if and only if  $\chi_{k,G,\bar v}^{(d+1)}=\chi_{k,G,\bar w}^{(d+1)}$.  We say that \emph{$\wlk{k}$ distinguishes two graphs~$G$ and~$H$} if the multisets of labels
 for all $k$-tuples of vertices in $G$ and $H$, respectively, coincides. Similar notions
 as are place for distinguishing $k$-tuples, and for distinguishing graphs (or vertices)
 based on labels computed by a given number of rounds.

 We remark that $\wlk{k}$ algorithm given here is sometimes referred to as the ``folklore'' version of the $k$-dimensional Weisfeiler-Leman algorithm. It is known that indistinguishability of graphs by $\wlk{k}$ is equivalent to indistinguishability by sentences in the the $k+1$-variable fragment of first order logic with counting \citep{CaiFI92}, and to $\homc{P,G}=\homc{P,H}$ for every graph of treewidth $k$ \citep{dvorak,DellGR18}.

\subsection{Proof of Proposition~\ref{prop:lgp-in-kwl}}
We show that for each finite set $\mathcal F$ of patterns, the expressive power of
$\wlk{\mathcal F}$  is bounded by $\wlk{k}$, where $k$ is the largest treewidth of a pattern in $\mathcal F$.

We first recall the following characterisation of $\wlk{k}$-equivalence \citep{dvorak,DellGR18}. For any two graphs $G$ and $H$,
$$
G\equiv_{\wlk{k}} H \Longleftrightarrow \homc{P,G}=\homc{P,H}
$$
for every graph $P$ of treewidth at most $k$. On the other hand, we know from Theorem~\ref{thm:char} that 
$$
G\equiv_{\wlk{\mathcal F}} H \Longleftrightarrow  \homc{T,G}=\homc{T,H}
$$
for every $\mathcal F$-pattern tree $T$. Hence, we may conclude that 
$$G\equiv_{\wlk{k}} H \Longrightarrow G\equiv_{\wlk{\mathcal F}} H$$
if we can show that any $\mathcal F$-pattern tree has treewidth at most $k$.

Suppose that $k$ is the maximal treewidth of a pattern in $\mathcal F$. To conclude the proof, we verify that the treewidth of any $\mathcal F$-pattern tree is bounded by $k$.

\begin{lemma}
	If $k$ is the maximal treewidth of a pattern in $\mathcal F$, then the treewidth of any $\mathcal F$-pattern tree $T$ is bounded by $k$.
\end{lemma}
\begin{proof}
The proof is by induction on the number of patterns joined at any leaf of $T$. Clearly, if no patterns are joined, then $T$ is simply a tree and its treewidth is $1$.
Otherwise, consider a $\mathcal F$-pattern tree $T = (V,E,\chi)$ whose treewidth is at most $k$, and a pattern $P^r$ of treewidth $k$ that is to be joined at vertex $t$ of $T$. By the induction hypothesis, there is a decomposition $(H,\lambda)$ for $T$ witnessing its bounded treewidth, that is, 
\begin{enumerate}\setlength{\itemsep}{0pt}
  \setlength{\parskip}{0pt}
\item The union of all $\lambda(h)$, for $h\in V_H$, is equal to $V$;
\item The set $\{h \in V_H \mid t \in \lambda(h)\}$ is connected, for all $t \in V$;
\item For each $\{u,v\} \in E$ there is $h \in V_H$ with $\{u,v\} \in \lambda(h)$; and
\item The size of each set $\lambda(h)$ is at most $k+1$.
\end{enumerate}
Likewise, by assumption, for pattern $P^r$ we have such a tree decomposition, say $(H^P, \lambda^P)$.

Now consider any vertex $h$ of the decomposition of $T$ such that $\lambda(h)$ contains vertex $t$ in $T$ to which $P^r$ is to be joined at its root. 
We can create a joint tree decomposition for the join of $P^r$ and $T$ (at node $t$) by merging $H$ and $H^P$ with an edge from vertex $h$ in $H$ to the root of $H^P$ (recall $H^P$ is a tree by definition). It is readily verified that this decomposition maintains all necessary properties. Indeed, condition 1 is clearly satisfied since $\lambda$ and $\lambda^p$ combined cover all vertices of the join of $T$ with $P^r$. Furthermore,
since the only node shared by $T$ and $P^r$ is the join node, and we merge $H$ and $H^P$ by 
putting an edge from node $h$ in $H$ to the root of $H^P$, connectivity of is guaranteed and condition 2 is satisfied.
Moreover, since the operation of joining $T$ and $P^r$ does not create any extra edges, condition 2 is 
immediately verified, and so is 3, because we do not create any new vertices, neither in $H$ nor in $H^P$, and we already know that $\lambda$ and $\lambda^P$ are bounded by $k+1$.
\end{proof}

\subsection{Proof of Theorem~\ref{theo:games}}
We show that if $\mathcal F$ contains a pattern $P^r$ which is a core and has treewidth $k$, then $\wlk{\mathcal F}$ is not bounded by  $\wlk{(k-1)}$. In other words, we construct two graphs $G$
and $H$ that can be distinguished by $\wlk{\mathcal F}$ but not by $\wlk{(k-1)}$. It suffices to find such graphs that can be distinguished by $\wlk{\{P^r\}}$ but not by $\wlk{(k-1)}$. The proof relies on the characterisation of $\wlk{(k-1)}$ indistinguishability in terms of the $k$-variable fragment $\mathsf{C}^{k}$ of first logic with counting and of $k$-pebble bijective games in particular \citep{CaiFI92,Hella96}. More precisely, $G\equiv_{\wlk{(k-1)}} H$ if and only if no sentence in $\mathsf{C}^k$ can distinguish $G$ from $H$. In other words, for any sentence $\varphi$ in $\mathsf{C}^k$, $G\models\varphi$ if and only if $H\models\varphi$. We denote indistinguishability by $\mathsf{C}^k$ by $G\equiv_{\mathsf{C}^k} H$. We heavily rely on the constructions used in \citet{AtseriasBD07} and \citet{BovaC19}.
In fact, we show that the graphs $G$ and $H$ constructed in those works, suffice for our purpose, by extending their strategy for the $k$-pebble game to $k$-pebble bijective games.

\paragraph{\boldmath Construction of the graphs $G$ and $H$.}
Let $P^r$ be a pattern in $\mathcal F$ which is a core
and has treewidth $k$. 
For a vertex $v\in V_P$, we gather all its edges in  $E_v:=\bigl\{\{v,v'\}\mid \{v,v'\}\in E_P\bigr\}$.
Let $v_1$ be one of the vertices in $V_P$.

For $G$, as vertex set $V_G$ we take vertices of the form $(v,f)$ with $v\in V_P$ and 
$f:E_v\to\{0,1\}$. We require that 
$$    \sum_{e\in E_{v_1}} f(e) \!\!\mod 2=1 \text{ and }
    \sum_{e\in E_v} f(e) \!\! \mod 2=0
\text{ for $v\neq v_1$, $v\in V_P$}.$$

For $H$, as vertex set $V_H$ we take vertices of the form
$(v,f)$ with $v\in V_P$ and $f:E_v\to\{0,1\}$.
We require that 
$
    \sum_{e\in E_v} f(e) \mod 2=0
$, for all $v\in V_P$.
We observe that $G$ and $H$ have the same number of vertices.

The edge sets $E_G$ and $E_H$ of $G$ and $H$, respectively, are defined as follows:
$(v,f)$ and $(v',f')$ are adjacent if and only if $v\neq v'$ and furthermore,
$$
f(\{v,v'\}=f'(\{v,v'\}).
$$

It is known that $\homc{P,G}=0$ (here it is used that $P$ is a core), $\homc{P,H}\neq 0$ and
$G$ and $H$ are indistinguishably by means of sentences in 
the $k$-variable fragment $\mathsf{FO}^k$ of first order logic 
\citep{AtseriasBD07,BovaC19}. To show our theorem, we thus
need to verify that $G\equiv_{\mathsf{C}^k} H$ as well. Indeed, for if this holds,
then $G\equiv_{\wlk{(k-1)}} H$ yet $G\not\equiv_{\wlk{\{P\}}}^{(0)} H$.
Indeed, Theorem~\ref{thm:char} implies that for $G\not\equiv_{\wlk{\{P\}}}^{(0)} H$ to hold, $\homc{P,G}=\homc{P,H}$, which we know not to be true. Hence,
$G\not\equiv_{\wlk{\{P\}}} H$, as desired.

\paragraph{\boldmath Showing  $\mathsf{C}^{k}$-indistinguishability of $G$ and $H$.}
We next show that the graphs $G$ and $H$ are indistinguishable by sentences in $\mathsf{C}^{k}$. This will be shown by verifying that the Duplicator has a winning strategy for the $k$-pebble bijective game on $G$ and $H$ \citep{Hella96}.

\smallskip
\underline{The $k$-pebble bijective game.}\ 
We recall that the $k$-pebble bijective game is played between two players, the Spoiler and the Duplicator, each placing at most $k$ pebbles on the vertices of $G$ and $H$, respectively. The game is played in a number of rounds.
The pebbles placed after round $r$ are typically represented by a partial function
$p^{(r)}:\{1,\ldots,k\}\to V_G\times V_H$.
When $p^{(r)}(i)$ is defined, say, $p^{(r)}(i)=(v,w)$, this means that the Spoiler places the $i$th pebble on vertex $v$ and the Duplicator places the $i$th pebble on $w$.
Initially, no pebbles are placed on $G$ and $H$ and hence $p^{(0)}$ is undefined everywhere.

Then in round $r>0$, the game proceeds as follows:
\begin{enumerate}\setlength{\itemsep}{0pt}
  \setlength{\parskip}{0pt}
    \item The Spoiler selects a pebble $i$ in $[k]$. All other already placed pebbles are kept on the same vertices. We define $p^{(r)}(j)=p^{(r-1)}(j)$ for all $j\in[k]$, $j\neq i$.
    \item The Duplicator responds by choosing a bijection $h:V_G\to V_H$. This bijection should be \textit{consistent} with the pebbles in the restriction of  $p^{(r-1)}$ to $[k]\setminus \{i\}$. That is, for every $j\in[k]$, $j\neq i$,    if $p^{(r-1)}(j)=(v,w)$ then $w=h(v)$.
    \item Next, the Spoiler selects an element $v\in V_G$.
    \item The Duplicator defines $p^{(r)}(i)=(v,h(v))$. Hence, after this round, the $i$th pebble is placed on $v$ by the Spoiler and on $h(v)$ by the Duplicator.
\end{enumerate}
Let $\mathsf{dom}(p^{(r)})$ be the elements in $[k]$ for which $p^{(r)}$ is defined. For $i\in\mathsf{dom}(p^{(r)})$ denote by $(v_i,w_i)\in V_G\times V_H$ the pair of vertices on which the $i$th pebble is placed.
The Duplicator \textit{wins} round $r$ if the mapping $v_i\mapsto w_i$ is partial isomorphism between $G$ and $H$. More precisely, it should hold that for all edges $\{v_i,v_j\}\in E_G$ if and only if $(w_i,w_j)\in E_H$. In this case, the game continues to the next round. Infinite games are won by the Duplicator. A winning strategy consists of defining a bijection in step 2 in each round, allowing the game to continue, irregardless of which vertex $v$ the Spoiler places a pebble in Step 3.

\smallskip
\underline{Winning strategy.}\ 
We will now provide a winning strategy for the $k$-bijective game on our constructed graphs $G$ and $H$. We recall that $V_G$ and $V_H$ have the same number of vertices, so a bijection between $V_G$ and $V_H$ exists. We show how the Duplicator can select a ``good'' bijection in Step 2 of the game, by induction on the number of rounds.

To state our induction hypothesis, we first recall some notions and properties from \citet{AtseriasBD07} and \citet{BovaC19}.

Let $W$ be a walk in $P$ and let $e$ be an edge in $E_P$. Then, $\mathsf{occ}_W(e)$ denotes the number of occurrences of the edge $e$ in the walk. More precisely, if $W=(a_1,\ldots,a_\ell)$ is a walk in $P$ of length $\ell$, then
$$
\mathsf{occ}_W(e):=|\{ i\in[\ell-1]\mid e=\{a_i,a_{i+1}\}\}|.
$$
Furthermore, for a subset $S\subseteq V_P$, we define  
$$
\mathsf{avoid}(S):=\bigcup_{\{M\in\mathcal{M},M\cap S=\emptyset\}} M,
$$
where $\mathcal{M}$ is an arbitrary bramble of $P$ of order $>k$. A bramble $\mathcal{M}$ is a set of \textit{connected} subsets of $V_P$ such that for any two elements $M_1$ and $M_2$ in $\mathcal{M}$, either $M_1\cap M_2\neq\emptyset$, or there exists a vertex $a\in M_1$ and $b\in M_2$ such that $\{a,b\}\in E_P$. The order of a bramble is the minimum size of a hitting set for $\mathcal{M}$. It is known that $P$ has treewidth $\geq k$ if and only if it has a bramble of order $>k$.  In what follows, we let $\mathcal M$ be any such bramble.

\begin{lemma}[Lemma 14 in \citet{BovaC19}]\label{lem:partiso}\normalfont 
For any $1\leq \ell\leq k$, let $(a_1,f_1),\ldots,(a_\ell,f_\ell)$ be vertices in $V_G$.
Let $W$ be a walk in $P$ from $v_1$ to $\mathsf{avoid}(\{a_1,\ldots,a_\ell\})$. For all $i\in[\ell]$,
let $f_i':E_{a_i}\to\{0,1\}$ be defined by 
$$
f'_i(e)=f_i(e) + \mathsf{occ}_{W}(e) \mod 2
$$
for all $e\in E_{a_i}$.
Then, the mapping $(a_i,f_i)\mapsto (a_i,f_i')$, for all $i\in[\ell]$, is a partial isomorphism from $G$ to $H$.\qed
\end{lemma}

We use this lemma to show that the bijection (to be defined shortly) selected by the Duplicator induces a partial isomorphism between $G$ and $H$ on the pebbled vertices.

We can now state our induction hypothesis: In each round
$r$, there exists a bijection $h:V_G\to V_H$ which is 
\begin{itemize}\setlength{\itemsep}{0pt}
  \setlength{\parskip}{0pt}
    \item[(a)]  consistent with the pebbles in the restriction of $p^{(r-1)}$ to $[k]\setminus \{i\}$ (Recall, Pebble $i$ is selected by the Spoiler in Step 1.)
    \item[(b)] If $p^{(r)}(j)=(a_j,f_j,h(a_j,f_j))$ for $j\in\mathsf{dom}(p^{(r)})$, then there exists a walk $W^{(r)}$ in $P$, from $v$ to $\mathsf{avoid}(\{a_j\mid j\in \mathsf{dom}(p^{(r)})\}$, such that
    $$
    h(a_j,f_j)=(a_j,f_j'),
    $$
    where $f_j'(e)=f_j(e) + \mathsf{occ}_{W^{(r)}}(e)\mod 2$ for every $e\in E_{a_j}$. In other words, on the vertices in $V_G$ pebbled by $p^{(r)}$, the bijection $h$ is, by the previous Lemma, a partial isomorphism from $G$ to $H$.
\end{itemize}
If this holds, then the strategy for the Duplicator is selecting that bijection $h$ in each round.

\smallskip
\underline{Verification of the induction hypothesis.}\ 
We assume that the special vertex $v_1$ in $P$ has at least two neighbours. Such a vertex exists since otherwise $P$ consists of a single edge while we assume $P$ to be of treewidth at least two.

\smallskip
\textit{Base case.}\ 
For the base case ($r=0)$ we define two walks:
$W_1=v_1,v_2$ and $W_2=v_1,t$ with $v_2\neq t$ and $v_2,t$ are neighhbours of $v_1$. We define
$h(a_i,f)=(a_i,f')$ with $f'(e)=f(e) +
\mathsf{occ}_{W_1}(e) \mod 2$ if $a_i\neq t$,
and 
$h(t,f)=(t,f')$ with $f'(e)=f(e)+\mathsf{occ}_{W_2}(e)\mod 2$.

The mapping $h$ is a bijection from $V_G$ to $V_H$. We note that it suffices to show that $h$ is injective since $V_G$ and $V_H$ contain the same number of vertices.
Since $h(a_i,f_i)\neq h(a_j,f_j)$ whenever $a_i\neq a_j$, we can focus on comparing $h(a_i,f)$ and $h(a_i,g)$ with $f\neq g$. This implies that $f(e)\neq g(e)$ for at least one edge $e\in N_{a_i}$. Clearly, this implies that
$f'(e)=f(e)+\mathsf{occ}_W(e)\mod 2\neq 
g'(e)=g(e)+\mathsf{occ}_W(e)\mod 2
$. In fact this, holds for any walk $W$ and thus in particular for $W_1$ and $W_2$.
We further observe that $h$ is consistent simply because no pebbles have been placed yet.
For the same reason we can take the walk $W^{(0)}$ to be either $W_1$ or $W_2$.

\smallskip
\textit{Inductive case.}\ 
Assume that the induction hypothesis holds for round $r$ and consider round $r+1$. Let $p^{(r)}=(a_j,f_j,a_j,f_j')$ for $j\in \mathsf{dom}(p^{(r)})$. By induction, there exists a bijection
$h':V_G\to V_H$ such that $h(a_j,f_j)=(a_j,f_j')$ and furthermore, 
$f_j'(e)=f_j(e)+ \mathsf{occ}_{W^{(r)}}(e)\mod 2$ for every $e\in N_{a_j}$,
for some walk $W^{(r)}$ from $v_1$ to $t\in\mathsf{avoid}(\{a_j\mid j\in\mathsf{dom}(p^{(r)})\})$.

Assume that the Spoiler selects $i\in[k]$ in Step 1 in round $r+1$. 
We define the Duplicator's bijection $h:V_G\to V_H$ for round $r+1$, as follows. Recall that $t\in V_P$
is the vertex in which the walk $W^{(r)}$ ends.
\begin{itemize}
    \item For all $(a,f)\in V_G$ such that $a\neq t$, we define $h(a,f)=(a,f')$ where for each
    $e\in E_a$:
    $$f'(e)=f(e)+\mathsf{occ}_{W^{(r)}}(e)\mod 2.$$
    \item For all $(t,f)\in V_G$, we will extend $W^{(r)}$ with a walk $W'$ so that it ends in a vertex $t'$ different from $t$. Suppose that $M\in \mathcal M$ such that $t\in M$. We want to find an $M'\in\mathcal M$ such that $M'\cap (\{a_j\mid j\in\mathsf{dom}(p^{(r)}),j\neq i\}\cup \{t\})=\emptyset$. We can then take $t'$ to be a vertex in $M'$ and since $M$ and $M'$ are both connected, and either have a vertex in common or an edge between them, we can let $W'$ be a walk from $t$ to $t'$ entirely in 
	$M$ and $M'$. Now, such an $M'$ exists since otherwise $\{a_j\mid j\in\mathsf{dom}(p^{(r)}),j\neq i\}\cup \{t\}$ would be a hitting set for $\mathcal M$ of size at most $k$. We know, however, that any hitting set $\mathcal M$ must be of size $k+1$ because of the treewidth $k$ assumption for $P$.
	We now define the bijection as
    $h(t,f)=(t,f')$ where for each $e\in E_t$:
    $$f'(e)=f(e)+\mathsf{occ}_{W^{(r)},W'}(e)\mod 2.$$
\end{itemize}
This concludes the definition of $h:V_G\to V_H$. We need to verify a couple of things: (i)~$h$ is bijection; (ii)~$h$ is consistent with all pebbles in $p^{(r)}$ except for the ``unpebbled'' one $p^{(r)}(i)$; and (iii) it induces a partial isomorphism on pebbled vertices.

\begin{itemize}
\item[(i)] $h$ is a bijection.
Since $V_G$ and $V_H$ are of the same size,
it suffices to show that $h$ is an injection. Clearly, $h(a_1,f_1)\neq h(a_2,f_2)$ whenever $a_1\neq a_2$. We can thus focus on $h(a,f_1)$ and $h(a,f_2)$ with $f_1\neq f_2$.
Then, $f_1$ and $f_2$ differ in at least one edge
$e\in E_a$ and for this edge:
$$
f_1'(e)=f_1(e)+\mathsf{occ}_W(e) \mod 2\neq f_2(e)+\mathsf{occ}_W(e) \mod 2=f_2'(e).
$$
for any walk $W$. In particular, this holds for both walks used in the definition of $h$: $W^{(r)}$, used when $a\neq t$, and
$W^{(r)},W'$ used when $a=t$.  Hence, $h$ is indeed a bijection.

\item[(ii)] $h$ is consistent.
For each $j\in\mathsf{dom}(p^{(r+1)})$ with $j\neq i$, let $p^{(r+1)}=(a_j,f_j,a_j,f_j')$. Now, by induction, $W^{(r)}$ ended in a vertex $t$ distinct from any of these $a_j$'s
and thus none of these $a_j$'s are equal to $t$. This implies that $h(a_j,f_j)=(a_j,f_j'')$ with 
$f_j''(e)=f_j(e) + \mathsf{occ}_{W^{(r)}}(e)\mod 2$. But this is precisely how $p^{(r)}$ placed its pebbles, by induction. Hence, $f_j''(e)=f_j'(e)$ and thus $h$ is consistent.

\item[(iii)] $p^{(r+1)}$ induces a partial isomorphism.
After the Spoiler picked an element $(a_i,f_i)\in V_G$,
we now know that $p^{(r+1)}(j)=(a_j,f_j,h(a_j,f_j))$ for
all $j\in\mathsf{dom}(p^{(r+1)})$. We recall that $h$ is defined in two possible ways, using two distinct walks:
$W^{(r)}$, for vertices in $V_G$ not involving $t$, or, otherwise using the walk $W^{(r)},W'$, for vertices in $V_G$ involving $t$.

Hence, when all $a_j$'s for $p^{(r+1)}$ are distinct from $t$, then $h(a_j,f_j)=(a_j,f_j')$ with
$f_j'(e)=f_j(e) + \mathsf{occ}_{W^{(r)}}(e)\mod 2$
and we can simply take the new walk $W^{(r+1)}$ to be $W^{(r)}$. Then, Lemma~\ref{lem:partiso} implies that the mapping $(a_j,f_j)\to h(a_j,f_j)$, for $j\in\mathsf{dom}(p^{(r+1)})$ is a partial isomorphism from $G$ to $H$, as desired.

Otherwise, we know that $a_j\neq t$ for $j\neq i$ but $a_i=t$. That is, the Spoiler places the $i$th pebble on a vertex of the form $(t,f)$ in $V_G$. We now have that
$h$ is defined in two ways for the pebbled elements using the two distinct walks.
We next show that $W^{(r)},W'$ can be used for both types of pebbled elements 
in $p^{(r+1)}$, those of the form $(a_j,f)$ with $a_j\neq t$ and $(t,f)$. For the
last type this is obvious since we defined $h(t,f)$ in terms of $W^{(r)},W'$.
For the former type, we note that $a_j\not\in M$ and $a_j\not\in M'$ for $j\neq i$.
If we take an edge $e\in N_{a_j}$, then 
$\mathsf{occ}_{W^{r},W'}(e)=\mathsf{occ}_{W^{(r)}}(e)$ because $W'$ lies entirely in $M$
and $M'$. As a consequence, for $(a_j,f_j)$ with $j\neq i$, for all $e\in N_j$:
\begin{align*}
f'_j(e)&=f_j(e) + \mathsf{occ}_{W^{(r)}}(e) \mod 2\\
& =f_j(e) + \mathsf{occ}_{W^{(r)},W'}(e) \mod 2.
\end{align*}
Then, Lemma~\ref{lem:partiso} implies that the mapping $(a_j,f_j)\to h(a_j,f_j)$, for $j\in\mathsf{dom}(p^{(r+1)})$ is a partial isomorphism from $G$ to $H$, because we can use the same walk $W^{(r),W'}$ for all pebbled vertices.
\end{itemize}

\subsection{Proof of Proposition~\ref{prop:long-cycles}}
We show that no finite set $\mathcal F$ of patterns suffices for $\wlk{\mathcal F}$ to be equivalent to $\wlk{k}$, for $k>1$, in terms of expressive power.
The proof is by contradiction. That is, suppose that there exists a set $\mathcal F$ such that $G\equiv_{\wlk{\mathcal F}} H\Leftrightarrow G\equiv_{\wlk{k}} H$
for any two graphs $G$ and $H$. In particular, $G\equiv_{\wlk{\mathcal F}} H\Rightarrow G\equiv_{\wlk{k}} H$ and thus also  $G\equiv_{\wlk{\mathcal F}} H\Rightarrow G\equiv_{\wlk{2}} H$, since the $\wlk{2}$-test is upper bounded by any $\wlk{k}$-test for $k>2$. We argue that no finite set $\mathcal F$ exists satisfying  $G\equiv_{\wlk{\mathcal F}} H\Rightarrow G\equiv_{\wlk{2}} H$.

Let $m$ denote the maximum number of vertices of any pattern in $\mathcal F$.\footnote{Strictly speaking, we can use the diameter of any pattern in $\mathcal F$ instead, but it is 
easier to convey the proof simply by taking number of vertices.} Furthermore, consider graphs $G$ and $H$, where $G$ is the disjoint union of  $m+2$ copies of the cycle $C_{m+1}$, and $H$ is the union of $m+1$ copies of the cycle $C_{m+2}$. Note that $G$ and $H$ have the same number of vertices. 

We observe that any homomorphism from a pattern $P^r$ in $\mathcal F$ to $G^v$ or $H^w$, for vertices $v\in V_G$ and $w\in V_H$, maps $P^r$ to either a copy of $C_{m+1}$ (for $G$) or a copy
of $C_{m+2}$ (for $H$). Furthermore, any such homomorphism maps $P^r$ in a subgraph of $C_{m+1}$ or $C_{m+2}$, consisting of at most $m$ vertices. There is, however, a unique (up to isomorphism) subgraph of $m$ vertices in
$C_{m+1}$ and $C_{m+2}$. Indeed, such subgraphs will be a  path of length $m$. This implies that $\homc{P^r,G^v}=\homc{P^r,H^w}$ for any $v\in V_G$ and $w\in V_H$.
Since the argument holds for any pattern $P^r$ in $\mathcal F$, all vertices in $G$ and $H$ will have the same homomorphism count for patterns in $\mathcal F$.
Furthermore, since both $G$ and $H$ are regular graphs (each vertex has degree two), this implies that $\wlk{\mathcal F}$ cannot distinguish between 
$G$ and $H$. This is formalised in the following lemma. We recall that a $t$-regular graph is a graph in which every vertex has degree $t$.
\begin{lemma}\label{lem:regularFWL}
For any set $\mathcal F$ of patterns and any two $t$-regular  (unlabelled) graphs $G$ and $H$ such that $\homc{P^r,X^x}=\homc{P^r,Y^y}$ for $P^r\in\mathcal F$, $X,Y\in\{G,H\}$, $x\in V_X$ and $y\in V_Y$ holds, $G\equiv_{\wlk{\mathcal F}} H$.
\end{lemma}
\begin{proof}
	The lemma is readily verified by induction on the number $d$ of rounds of $\wlk{\mathcal F}$. We show a stronger result in that 
	$\chi_{\mathcal F,X,x}^{(d)}=\chi_{\mathcal F,Y,y}^{(d)}$ for any $d$, $X,Y\in\{G,H\}$, $x\in V_X$ and $y\in V_Y$, from which $G\equiv_{\wlk{\mathcal F}} H$ follows.
	By our Theorem~\ref{thm:char}, it suffices to show that
	$\homc{T^r,X^x}=\homc{T^r,Y^y}$ for  $\mathcal F$-pattern trees of depth at most $d$. Let $\mathcal F=\{P_1^r,\ldots,P_\ell^r\}$. For the base case, let $T^r$ be a join pattern $\mathcal F^{\mathbf{s}}$ for some $\mathbf{s}=(s_1,\ldots,s_\ell)\in\Nb^{\ell}$.
	Then,
	\allowdisplaybreaks
\begin{align*}
	\homc{T^r,X^x}=\prod_{i=1}^{\ell}\left(\homc{P_i^r,X^x}\right)^{s_i}=\prod_{i=1}^{\ell}\left(\homc{P_i^r,Y^y}\right)^{s_i}=\homc{T^r,Y^y},
\end{align*}
since  $\homc{P_i^r,X^x}=\homc{P_i^r,Y^y}$ for any $P_i^r\in\mathcal F$.
Then, for the inductive case, assume that $\homc{S^r,X^x}=\homc{S^r,Y^y}$ for any $\mathcal F$-pattern tree $S^r$ of depth at most $d-1$, $X,Y\in\{G,H\}$, $x\in V_X$ and $y\in V_Y$, and consider an $\mathcal F$-pattern $T^r$ of depth $d$. Let $S_1^{c_1},\ldots,S_p^{c_p}$ be the $\mathcal F$-pattern trees of depth at most $d-1$ rooted at the children $c_1,\ldots,c_p$ of $r$
in the backbone of $T^r$. As before, let $\mathcal F^{\mathbf{s}}$ the pattern joined at $r$ in $T^r$. Then,
	\allowdisplaybreaks
\begin{align*}
	\homc{T^r,X^x}&=\homc{\mathcal F^{\mathbf{s}},X^x}\prod_{i=1}^{p}\sum_{x'\in N_X(x)}\homc{S_i^{c_i},X^{x'}}=\homc{\mathcal F^{\mathbf{s}},X^x}\prod_{i=1}^{p} t\cdot \homc{S_i^{c_i},X^{\tilde{x}}}\\
	&=\homc{\mathcal F^{\mathbf{s}},Y^y}\prod_{i=1}^{p} t \cdot \homc{S_i^{c_i},Y^{\tilde{y}}}=\homc{\mathcal F^{\mathbf{s}},Y^y}\prod_{i=1}^{p}\sum_{y'\in N_Y(y)}\homc{S_i^{c_i},Y^{y'}}\\
		&=\homc{T^r,Y^y},
\end{align*}
where we used that $N_X(x)$ and $N_Y(y)$ both consists of $t$ vertices (regularity), by the induction hypothesis all vertex have the same homomorphism counts
for $\mathcal F$-patterns trees of depth at most $d-1$, and where $\tilde{x}$ and $\tilde{y}$ are taken to be arbitrary vertices in $N_X(x)$ and $N_Y(y)$, respectively.
\end{proof}
Hence, since $G$ and $H$ are $2$-regular and satisfy the conditions of the lemma, we may indeed infer that $G\equiv_{\wlk{\mathcal F}} H$. We note, however, that $G\not\equiv_{\wlk{2}} H$. Indeed, from \citet{dvorak} and \citet{DellGR18} we know that 
$G\equiv_{\wlk{2}} H$ implies that $\homc{P,G}=\homc{P,H}$ for any graph $P$ of treewidth at most two. In particular, 
$G\equiv_{\wlk{2}} H$ implies that $\homc{C_\ell,G}=\homc{C_\ell,H}$ for all cycles $C_\ell$. We now conclude by observing that 
$\homc{C_{m+1},G}\neq \homc{C_{m+1},H}$ by construction.
We have thus found two graphs with cannot be distinguished by $\wlk{\mathcal F}$, but that can be distinguished by $\wlk{2}$, contradicting our assumption
that $G\equiv_{\wlk{\mathcal F}} H\Rightarrow G\equiv_{\wlk{2}} H$.

\section{Proofs of Section~\ref{sec:patterns}}

\subsection{Proof of Proposition~\ref{prop:cycles}}
We show that for any $k > 3$,  $\wlk{\{C^r_3,\dots,C^r_k\}}$ is more expressive than 
$\wlk{\{C^r_3,\dots,C^r_{k-1}\}}$. More precisely, we construct two graphs $G$ and $H$
such that $G$ and $H$ cannot be distinguished by $\wlk{\{C^r_3,\dots,C^r_{k-1}\}}$, but they can be distinguished by  $\wlk{\{C^r_3,\dots,C^r_k\}}$.

The proof is analogous to the proof of Proposition~\ref{prop:long-cycles}. Indeed, it suffices to let $G$ consist of $k$ disjoint copies of $C_{k+1}$ and $H$ to consist of $k+1$ disjoint copies of $C_{k}$. Then, as observed in the proof of Proposition~\ref{prop:long-cycles}, $G$ and $H$ will be indistinguishable by 
$\wlk{\{C^r_3,\dots,C^r_{k-1}\}}$ simply because each pattern has at most $k-1$ vertices. Yet, by construction, $\homc{C_k,G}\neq \homc{C_k,H}$ and thus $G$ and $H$ are distinguishable (already by the initial labelling) by  $\wlk{\{C^r_3,\dots,C^r_k\}}$.

\subsection{Proof of Proposition~\ref{prop:simplify}}
Let $P^r = P_1^r \star P_2^r$ be a pattern that is the join of two smaller patterns. 
We show that for any  any set $\mathcal F$ of patterns, we have that 
$\wlk{\mathcal F \cup \{P^r\}}$ is upper bounded by  $\wlk{\mathcal F \cup \{P_1^r,P_2^r\}}$.
That is, for every two graphs $G$ and $H$, $G\equiv_{\wlk{\mathcal F \cup \{P_1^r,P_2^r\}}} H$ implies $G\equiv_{\wlk{\mathcal F \cup \{P^r\}}} H$. 
By definition,  $G\equiv_{\wlk{\mathcal F \cup \{P_1^r,P_2^r\}}} H$  is equivalent to $\ldbl\chi_{\mathcal F\cup \{P_1^r,P_2^r\},G,v}^{(d)}\mid v\in V_G\rdbl=\ldbl\chi_{\mathcal F\cup \{P_1^r,P_2^r\},H,w}^{(d)}\mid w\in V_H\rdbl$. In other words, with every $v\in V_G$
we can associate a unique $w\in V_H$ such that $\chi_{\mathcal F\cup \{P_1^r,P_2^r\},G,v}^{(d)}=\chi_{\mathcal F\cup \{P_1^r,P_2^r\},H,w}^{(d)}$. We show, by induction on $d$, that this implies that $\chi_{\mathcal F\cup \{P^r\},G,v}^{(d)}=\chi_{\mathcal F\cup \{P^r\},H,w}^{(d)}$. This suffices to conclude that $\ldbl\chi_{\mathcal F\cup\{P^r\},G,v}^{(d)}\mid v\in V_G\rdbl=\ldbl\chi_{\mathcal F\cup\{P^r\},H,w}^{(d)}\mid w\in V_H\rdbl$ and thus $G\equiv_{\wlk{\mathcal F \cup \{P^t\}}} H$.

\smallskip
\noindent
\underline{Base case.} We show that $\ldbl\chi_{\mathcal F\cup \{P_1^r,P_2^r\},G,v}^{(d)}\mid v\in V_G\rdbl=\ldbl\chi_{\mathcal F\cup \{P_1^r,P_2^r\},H,w}^{(d)}\mid w\in V_H\rdbl$ implies that with every $v\in V_G$ we can associate a unique $w\in V_H$ satisfying $\chi_{\mathcal F\cup \{P^r\},G,v}^{(0)}=\chi_{\mathcal F\cup \{P^r\},H,w}^{(0)}$. Indeed, as already observed, 
$\ldbl\chi_{\mathcal F\cup \{P_1^r,P_2^r\},G,v}^{(d)}\mid v\in V_G\rdbl=\ldbl\chi_{\mathcal F\cup \{P_1^r,P_2^r\},H,w}^{(d)}\mid w\in V_H\rdbl$ implies that with every $v\in V_G$
we can associate a unique $w\in V_H$ such that $\chi_{\mathcal F\cup \{P_1^r,P_2^r\},G,v}^{(d)}=\chi_{\mathcal F\cup \{P_1^r,P_2^r\},H,w}^{(d)}$.
 This in turn implies that $\chi_{\mathcal F\cup \{P_1^r,P_2^r\},G,v}^{(0)}=\chi_{\mathcal F\cup \{P_1^r,P_2^r\},H,w}^{(0)}$, which implies  that $\homc{P_1^{r},G^v}=\homc{P_1^{r},H^w}$ and 
$\homc{P_2^{r},G^v}=\homc{P_2^{r},H^w}$ and $\homc{Q^r,G^v}=\homc{Q^r,H^w}$ for every $Q^r\in \mathcal F$.  As a consequence, from properties of the graph join operators, since $P^r=P_1^r\star P_2^r$:
$$
\homc{P^r,G^v}=\homc{P_1^r,G^v}\cdot\homc{P_2^r,G^v}=\homc{P_1^r,H^w}\cdot\homc{P_2^r,H^w}=\homc{P^r,H^w},
$$
and thus also  $\chi_{\mathcal F\cup \{P^r\},G,v}^{(0)}=\chi_{\mathcal F\cup \{P^r\},H,w}^{(0)}$.

\smallskip
\noindent
\underline{Inductive case.}
We assume that $\ldbl\chi_{\mathcal F\cup \{P_1^r,P_2^r\},G,v}^{(d)}\mid v\in V_G\rdbl=\ldbl\chi_{\mathcal F\cup \{P_1^r,P_2^r\},H,w}^{(d)}\mid w\in V_H\rdbl$ implies
$\chi_{\mathcal F\cup \{P^r\},G,v}^{(e)}=\chi_{\mathcal F\cup \{P^r\},H,w}^{(e)}$, and want to show that it also implies $\chi_{\mathcal F\cup \{P^r\},G,v}^{(e+1)}=\chi_{\mathcal F\cup \{P^r\},H,w}^{(e+1)}$. We again use the fact that we can associate with every $v\in V_G$ a unique
vertex $w\in V_H$ such that $\chi_{\mathcal F\cup \{P_1^r,P_2^r\},G,v}^{(d)}=\chi_{\mathcal F\cup \{P_1^r,P_2^r\},H,w}^{(d)}$. In particular, this implies that 
$\chi_{\mathcal F\cup \{P_1^r,P_2^r\},G,v}^{(e)}=\chi_{\mathcal F\cup \{P_1^r,P_2^r\},H,w}^{(e)}$ and 
$\chi_{\mathcal F\cup \{P_1^r,P_2^r\},G,v}^{(e+1)}=\chi_{\mathcal F\cup \{P_1^r,P_2^r\},H,w}^{(e+1)}$. From the definition of the $\wlk{}$-test, it must also be the case that the multisets 
$\{\chi_{\mathcal F\cup \{P_1^r,P_2^r\},G,v'}^{(e)} \mid v' \in N_G(v)\}$ and 
$\{\chi_{\mathcal F\cup \{P_1^r,P_2^r\},H,w'}^{(e)} \mid v' \in N_H(w)\}$ must be equal as well, i.e., we can find a one-to-one corresponence 
between neighbors of $v$ in $G$ and neighbors of $w$ in $H$ that have the same label. 
From the induction hypothesis we then have that $\chi_{\mathcal F\cup \{P^r\},G,v}^{(e)}=\chi_{\mathcal F\cup \{P^r\},H,w}^{(e)}$ and also that the multisets 
$\{\chi_{\mathcal F\cup \{P^r\},G,v'}^{(e)} \mid v' \in N_G(v)\}$ and 
$\{\chi_{\mathcal F\cup \{P^r\},H,w'}^{(e)} \mid v' \in N_H(w)\}$ are equal, which implies, by the definition of the $\wl{}$-test, that 
$\chi_{\mathcal F\cup \{P^r\},G,v}^{(e+1)}=\chi_{\mathcal F\cup \{P^r\},H,w}^{(e+1)}$, as was to be shown.

\subsection{Proof of Theorem~\ref{thm:increase}}
We show that $\wlk{\mathcal F\cup \{Q^r\}}$, where $Q^r$ is pattern whose core has treewidth $k$, is more expressive than 
$\wlk{\mathcal F}$ if every pattern $P^r\in\mathcal F$
satisfies one of the following conditions: (i)~$P^r$  has treewidth $< k$; or (ii)~
$P^r$ does not map homomorphically to $Q^r$.

Let $c(Q)^r$ to denote the (rooted) core of $Q$, in which the root of $c(Q)^r$ is any vertex which is 
the image of the root of $Q^r$ in a homomorphism from $Q^r$ to $c(Q)^r$. By assumption, $c(Q)^r$ has treewidth $k$. 

Clearly, $\wlk{\mathcal F}$ is upper bounded by $\wlk{\mathcal F\cup \{Q^r\}}$. Thus, all we need for the proof is to find two graphs that 
are indistinguishable by $\wlk{\mathcal F}$ but are in fact distinguished by $\wlk{\mathcal F\cup \{Q^r\}}$. 

Those two graphs are, in fact, the graphs $G$ and $H$ constructed for $c(Q)^r$ (of treewidth $k$) in the proof of Theorem \ref{theo:games}.
From that proof, we know that:
\begin{enumerate}\setlength{\itemsep}{0pt}
  \setlength{\parskip}{0pt}
\item[(a)] $\homc{c(Q),G}=0$ and $\homc{c(Q),H}\neq 0$; and
\item[(b)] $G\equiv_{\mathsf{C}^k} H$.
\end{enumerate}
We note that (a) immediately implies that $G$ and $H$ can be distinguished by $\wlk{\mathcal F\cup \{Q^r\}}$. In fact, they are distinguished in already by the initial labelling in round $0$. 
We next show that $G$ and $H$ are indistinguishable by $\wlk{\mathcal F}$.

Let us first present a small structural result that helps us deal with patterns in $\mathcal F$ satisfying the second condition of the Theorem. 
\begin{lemma}
\label{lem-getrid-p}
If a rooted pattern $P^r$  does not map homomorphically to $Q^r$, then $\homc{P,G}=\homc{P,H}=0$
\end{lemma}

\begin{proof}
We use the following property of graphs $G$ and $H$, which can be directly observed from their construction (and was already noted in \citet{AtseriasBD07} and \citet{BovaC19}). 
Define $G^r$ and $H^r$ by setting as their root any vertex $(a_r,f)$, for $a_r$ the root of $c(Q)^r$. 
Then there is a homomorphism from $G^r$ to $c(Q)^r$, and there is a homomorphism from $H^r$ to $c(Q)^r$. 

Now, any homomorphism $h$ from $P^r$ to $G$ can be extended to a homomorphism from $P^r$ to $Q^r$: we compose $h$ with the homomorphism mentioned above from $G$ to $c(Q)^r$, which by definition again maps homomorphically to $Q^r$. Since by definition we have that $P^r$ does not map to $Q^r$, $h$ cannot exist. The proof for $H$ is analogous.  
\end{proof}

Now, let $\mathcal F'$ be the set of patterns obtained by removing from $\mathcal F$ all patterns which do not map homomorphically to $Q^r$. By Lemma 
\ref{lem-getrid-p}, we have that $G$ and $H$ are distinguished by the $\wlk{\mathcal F}$-test if and only if they are distinguished by $\wlk{\mathcal F'}$. 

But all patterns in $\mathcal F'$ must have treewidth less than $k$, and by (b) $G$ and $H$ are indistinguishable by $\wlk{k}$.
Proposition~\ref{prop:lgp-in-kwl} then implies that $G$ and $H$ are indistinguishable by $\wlk{\mathcal F}$, as desired.

\section{Connections to related work}
We here provide more details of how $\mathcal F$-$\MPNNs$ connect to $\MPNNs$ from the literature which also augment the initial labelling.

\paragraph{Vertex degrees.}
We first consider so-called \textit{degree-aware} $\MPNNs$ \citep{geerts2020lets}
in which the message functions of the $\MPNNs$ may depend on the vertex degrees. The Graph Convolution Networks ($\mathsf{GCN}$s) \citep{kipf-loose} are an example of such $\MPNNs$. Degree-aware $\MPNNs$ are known to be equivalent, in terms of expressive power, to standard $\MPNNs$ in which  the initial labelling is extended with vertex degrees \citep{geerts2020lets}. Translated to our setting, we can simply let $\mathcal{F}=\{\raisebox{-1.1\dp\strutbox}{\includegraphics[height=3ex]{L1.pdf}}\}$
since $\homc{\raisebox{-1.1\dp\strutbox}{\includegraphics[height=3ex]{L1.pdf}},G^v}$ is equal to the degree of vertex $v$ in $G$.  When considering graphs without an initial vertex labelling (or a uniform labelling which assigns every vertex the same label), our characterisation (Theorem~\ref{thm:char}) implies $G\equiv_{\wlk{\raisebox{-0.5\dp\strutbox}{\includegraphics[height=2ex]{L1.pdf}}}}^{(d)} H$ if and only if $\homc{T,G}=\homc{T,H}$ for every $\{\raisebox{-1.1\dp\strutbox}{\includegraphics[height=3ex]{L1.pdf}}\}$-pattern tree of depth at most $d$. This in turn is equivalent to  $\homc{T,G}=\homc{T,H}$ for every (standard) tree of depth at most $d+1$. Indeed, $\{\raisebox{-1.1\dp\strutbox}{\includegraphics[height=3ex]{L1.pdf}}\}$-pattern trees of depth at most $d$ are simply trees of depth $d+1$. Combining this with the characterisation of $\wl{}$ by \citet{dvorak} and \citet{DellGR18}, we thus have for unlabelled  graphs that $G\equiv_{\wlk{\raisebox{-0.5\dp\strutbox}{\includegraphics[height=2ex]{L1.pdf}}}}^{(d)} H$ if and only if $G\equiv_{\wl{}}^{(d+1)} H$. So, by considering $\mathcal{F}=\{\raisebox{-1.1\dp\strutbox}{\includegraphics[height=3ex]{L1.pdf}}\}$-$\MPNNs$ one gains one round of computation compared to considering standard $\MPNNs$. To lift this to labeled graphs, instead of  $\mathcal{F}=\{\raisebox{-1.1\dp\strutbox}{\includegraphics[height=3ex]{L1.pdf}}\}$
one has to include labeled versions of the single edge pattern, in order to count the number of neighbours of a specific label for each vertex. This is done, e.g., by
 \citet{ishiguro2020weisfeilerlehman}, who use the $\wl{}$ labelling obtained after the first round to augment the initial vertex labelling. This corresponds indeed by adding $\homc{T^r,G^v}$ as feature for every labeled tree of depth one.
 This results in that $G\equiv_{\wlk{\raisebox{-0.5\dp\strutbox}{\includegraphics[height=2ex]{L1.pdf}}}}^{(d)} H$ if and only if $G\equiv_{\wl{}}^{(d+1)} H$ for labelled graphs.

\paragraph{Walk counts.} The \textit{Graph Feature Networks} by \citet{chen2019powerful} can be regarded as a generalisation of the previous approach. Instead of simply adding vertex degrees, the number of walks of certain
lengths emanating from vertices are added. Translated to our setting, this corresponds to considering $\{L_2,L_3,\ldots,L_\ell\}$-$\MPNNs$, where $L_\ell$ denotes a rooted path of length $\ell$. For unlabelled graphs, our characterisation (Theorem~\ref{thm:char}) implies that $G\equiv_{\wlk{L_1,\ldots,L_\ell}}^{(d)} H$
is upper bounded by  $G\equiv_{\wl{}}^{(d+\ell)} H$, simply because every 
$\{L_2,L_3,\ldots,L_\ell\}$-pattern tree of depth $d$ is a standard tree of depth
at most $d+\ell$.

\paragraph{Cycles.} \citet{li2019hierarchy} extend $\MPNNs$ by varying the notion of neighbourhood over which is aggregated. One particular instance corresponds to an aggregation of features, weighted by the number of cycles of a certain length in each vertex (see discussion at the end of Section 4 in  \citet{li2019hierarchy}). Translated to our setting,
this corresponds to considering $\{C_\ell\}$-$\MPNNs$ where $C_\ell$ denotes the cycle of length $\ell$. As mentioned in the main body of the paper, these extend $\MPNNs$ and result in architectures bounded by $\wlk{2}$ (Proposition~\ref{prop:cycles}). This is in line with Theorem 3 from \citet{li2019hierarchy} stating that their framework strictly extends $\MPNNs$ and thus $\wlk{1}$.

\paragraph{Isomorphism counts.}
Another, albeit similar, approach to add structural information to the initial labelling is taken in the paper \textit{Graph Structure Networks} by  \citet{bouritsas2020improving}. The idea there is to extend the initial features with information about how often a vertex $v$ appears in a subgraph of $G$ which is isomorphic to $P$. More precisely,
 \citet{bouritsas2020improving} consider a connected unlabelled graph $P$ as pattern
 and partition its vertex set $V_P$ orbit-wise. That is, $V_P=\biguplus_{i=1}^{o_P} V_P^i$ where $o_P$ denotes the number of orbits of $P$. Here, $v,v'\in V_P^i$ whenever there is an automorphism $h$ in $\mathsf{Aut}(P)$
 mapping $v$ to $v'$. Next, they consider all distinct subgraphs $G_1,\ldots,G_k$ in $G$
 which are isomorphic to $P$, denoted by $P\cong G_j$ for $j\in[k]$. We write
 $P\cong_f G_j$ when $P\cong G_j$ using a specific isomorphism $f$.
  Then for each orbit partition $i\in[o_P]$ and vertex $v\in V$, they define:
 $$
 \mathsf{iso}(P,G,v,i)=|\{ G_j\cong P\mid v\in V_{G_j}, \text{and there exists an $f$ s.t. } G_j\cong_f P \text{ and } f(v)\in V_P^i,j\in[k]\}|.
 $$
 That is, the number of subgraphs $G_j$ in $G$ that can be isomorphically mapped to $P$ are counted, provided that this can be done by an isomorphism which maps vertex $v$ in $G_j$ (and thus $G$) to one of the vertices in the $i$th orbit partition $V_P^i$ of the pattern. A similar notion is proposed for edges, which we will not consider here. Similar to our extended features, the initial features of each vertex $v$ is then augmented with 
 $\big( \mathsf{iso}(P,G^v,i)\mid P\in\mathcal F, i\in [o_P]\big)$
 for some set $\mathcal F$ of patterns. Standard $\MPNNs$ are executed on these augmented initial features. We refer to \citet{bouritsas2020improving} for more details.
 
We can view the above approach as an instance of our framework. Indeed, given a pattern $P$ in $\mathcal F$, for each orbit partition, we replace $P$ by a different rooted version $P^{r_i}$, where $r_i$ is a vertex in $V_{P}^i$. Which vertex in the orbit under consideration is selected as root is not important (because they are equivalent by definition of orbit). We then see that the standard notion 
of subgraph isomorphism counting directly translates to the quantity used in~\citet{bouritsas2020improving}:
 $$
 \mathsf{sub}(P^{r_i},G^v):=\text{number of subgraphs in $G$ containing $v$, isomorphic to $P^{r_i}$}= \mathsf{iso}(P,G,v,i).
 $$
It thus remains to express $\mathsf{sub}(P^{r_i},G^v)$ in terms of homomorphism counts. This, however, follows from \citet{Curticapean_2017} in which it is shown that $\mathsf{iso}(P^{r_i},G^v)$ can be computed by a linear combination of 
$\homc{Q^{r_i},G^v}$ where $Q^{r_i}$ ranges over all graphs on which $P^{r_i}$ can be mapped by means of  a surjective homomorphism. For a given $P^{r_i}$, the finite set of such patterns is called the \textit{spasm} of $P^{r_i}$ in  \citet{Curticapean_2017} and can be easily computed.

In summary, given the set $\mathcal F$ of patterns in \citet{bouritsas2020improving}, we first replace every $P\in\mathcal F$ by its rooted versions $P^{r_i}$, for $i\in[o_P]$, and then expand the resulting set of rooted patterns, by the spasms of each of these patterns. Let us denote by $\mathcal F^*$ the final set of rooted patterns. It now follows that
 $\homc{Q^r,G^v}$ for $Q^r\in \mathcal F^*$ provides sufficient information to extract
 $\mathsf{sub}(P^{r_i},G^v)$ and thus also $\mathsf{iso}(P,G,v,i)$ for every $P\in\mathcal F$ and orbit part $i\in[o_P]$. As a consequence, the $\MPNNs$ from \citet{bouritsas2020improving} are bounded by $\mathcal F^*$-$\MPNNs$ and thus $\wlk{\mathcal F^*}$. Conversely, given an $\mathcal F$-$\MPNN$ one can, again using results by \citet{Curticapean_2017}, define a set $\mathcal F^+$ of patterns, such that the subgraph isomorphism counts of patterns in $\mathcal F^+$ can be used to compute the homomorphism counts of patterns in $\mathcal F$. Hence, $\mathcal F$-$\MPNNs$ are upper bounded by the $\MPNNs$ considered in \citet{bouritsas2020improving} using patterns in $\mathcal F^+$. This is in agreement  with \citet{Curticapean_2017} in which it is shown that homomorphism counts, subgraph isomorphism counts and other notions of pattern counts are all interchangeable. Nevertheless, by using homomorphism counts one can gracefully extend known results about $\wl{}$ and $\MPNNs$, as we have shown in the paper, and add little overhead.

\section{Additional experimental information}

\subsection{Experimental setup}
One of the crucial questions when studying the effect of adding structural information to the initial vertex labels is whether these additional labels enhance the performance of graph neural networks. In order to reduce the effect of specific implementation details of $\GNNs$ and choice of hyper-parameters, we start from 
the $\GNN$ implementations and choices made in the benchmark by \citet{dwivedi2020benchmarkgnns}\footnote{The original implementations can be found on https://github.com/graphdeeplearning/benchmarking-gnns}. and only change the initial vertex labels, while leaving the $\GNNs$ themselves unchanged. This ensures that we only measure the effect of augmenting initial features with homomorphism counts. We will use the $\GNNs$ from the benchmark, without extended features, as our baselines. For the same reasons, we use datasets proposed in the benchmark for their ability to statistically separate the performance of $\GNNs$. All other parameters are taken as in \citet{dwivedi2020benchmarkgnns} and we refer to that paper for more details.

\paragraph{\boldmath Selected $\GNNs$} \citet{dwivedi2020benchmarkgnns} divide the benchmarked $\GNNs$ into two classes: the $\MPNNs$ and the ``theoretically designed'' $\mathsf{WL}$-$\GNNs$. The first class is found to perform stronger and train faster. Hence, we chose to include the five following $\MPNN$ models from the benchmark:

\begin{itemize}\setlength{\itemsep}{0pt}
  \setlength{\parskip}{0pt}
	\item Graph Attention Network ($\mathsf{GAT}$) as described in \citet{GAT}
	\item Graph Convolutional Network ($\mathsf{GCN}$) as described  in \citet{kipf-loose}
	\item $\mathsf{GraphSage}$ as described  in \cite{hyl17}
	\item Mixed Model Convolutional Networks ($\mathsf{MoNet}$) as  described  in \citet{monet}
	\item $\mathsf{GatedGCN}$ as described  in \citet{gated}.
\end{itemize}

For $\mathsf{GatedGCN}$ we used the version in which positional encoding \citep{LaplacianPE} is added to the vertex features, as it is empirically shown to be the strongest performing version of this model by for the selected datasets \citep{dwivedi2020benchmarkgnns}. We denote this version by $\mathsf{GatedGCN}_{E,PE}$, referring to the presence of edge features and this positional encoding. 
Details, background and a mathematical formalisation of the message passing layers of these models can be found in the supplementary material of \citet{dwivedi2020benchmarkgnns}. 

As explained in the experimental section of the main paper, we enhance the vertex features with the log-normalised counts of the chosen patterns in every vertex of every graph of every dataset. The first layers of some models of \citep{dwivedi2020benchmarkgnns} are adapted to take in this variation in input size. All other layers where left identical to their original implementation as provided by \citet{dwivedi2020benchmarkgnns}.
\paragraph{\boldmath Hardware, compute and resources} All models for ZINC, PATTERN and COLLAB were trained on a GeForce GTX 1080
337 Ti GPU, for CLUSTER a Tesla V100-SXM3-32GB GPU was used. Tables \ref{ZINC_time}, \ref{COLLAB_time}, \ref{CLUSTER_time} and \ref{PATTERN_time} report the training times for all combination of models and additional feature set. A rough estimate of the CO$_2$ emissions based on the total computing times of reported experiments ($2\,074$ hours GeForce GTX 1080, $372$ hours
 Tesla V100-SXM3-32GB), the computing times of not-included experiments
 ($1\,037$ hours GeForce GTX 1080, $181$ hours
  Tesla V100-SXM3-32GB), the GPU types (GeForce GTX 1080, Tesla V100-SXM3-32GB) and the geographical location of our cluster results in a carbon emission of $135$ kg CO$_2$ equivalent. This estimation was conducted using the \href{https://mlco2.github.io/impact#compute}{MachineLearning Impact calculator} presented in \citet{lacoste2019quantifying}.

\subsection{Graph learning tasks}
We here report the full results of our experimental evaluation for graph regression (Section~\ref{subsec:regression}), link prediction (Section~\ref{subsec:prediction}) and vertex classification (Section~\ref{subsec:classification}) as considered in  \citet{dwivedi2020benchmarkgnns}.
More precisely, a full listing of the patterns and combinations used and the obtained results for the test sets can be found in Tables \ref{ZINC_full}, \ref{COLLAB_full}, \ref{CLUSTER_full} and \ref{PATTERN_full}. Average training time (in hours) and the number of epochs are reported in  Tables \ref{ZINC_time}, \ref{COLLAB_time}, \ref{CLUSTER_time} and \ref{PATTERN_time}. Finally, the total number of model parameters are reported in Tables \ref{ZINC_par}, \ref{COLLAB_par}, \ref{CLUSTER_par} and \ref{PATTERN_par}. All averages and standard deviations are over 4 runs with different random seeds. The main take-aways from these results are included in the main paper.

\subsubsection{Graph regression with the ZINC dataset}\label{subsec:regression}
Just as in \citet{dwivedi2020benchmarkgnns} we use a subset ($12$K) of ZINC molecular graphs ($250$K) dataset \citep{zinc} to regress a molecular property known as the constrained solubility. For each molecular graph, the vertex features are the types of heavy atoms and the edge features are the types of bonds between them. The following are taken from \citet{dwivedi2020benchmarkgnns}:
\\
\textbf{Splitting}. ZINC has $10\,000$ train, $1\,000$ validation and $1\,000$ test graphs. \\
\textbf{Training}.\footnote{Here and in the next tasks we are using the parameters used in the code accompanying \citet{dwivedi2020benchmarkgnns}. In the paper, slightly different parameters are used.} For the learning rate strategy, an initial learning rate is set to $ 5 \times 10^{-5}$
, the
reduce factor is $0.5$, and the stopping learning rate is $ 1 \times 10^{-6}$, the patience value is 25 and the maximal training time is set to $12$ hours.\\
\textbf{Performance Measure} The performance measure is the mean absolute error (MAE) between the
predicted and the ground truth constrained solubility for each molecular graph.\\
\textbf{Number of layers}  16 MPNN layers are used for every model.

\begin{table}[hbt!]
	\caption{Full results of the mean absolute error (predicted constrained solubility vs. the ground truth) for selected cycle combinations and $\GNNs$ on the ZINC data set. In the last two rows we compare between homomorphism counts (hom) and subgraph isomorphism counts (iso).}
	\label{ZINC_full}
	\begin{center}
	\begin{tabular}{|l|l|l|l|l|l|}
		\hline
		Pattern set $\mathcal{F}$ & $\mathsf{GAT}$           & $\mathsf{GCN}$           & $\mathsf{GraphSage}$     & $\mathsf{MoNet}$         & $\mathsf{GatedGCN}_{E,PE}$    \\ \hline\hline
		None     & 0,47$\pm$0,02 & 0,35$\pm$0,01 & 0,25$\pm$0,01 & 0,44$\pm$0,01 & 0,34$\pm$0,05 \\ \hline
		$\{C_3\}$      & 0,45$\pm$0,01 & 0,36$\pm$0,01 & 0,25$\pm$0,00 & 0,44$\pm$0,00 & 0,30$\pm$0,01 \\ \hline
		$\{C_4\}$      & 0,34$\pm$0,02 & 0,29$\pm$0,02 & 0,26$\pm$0,01 & 0,30$\pm$0,01 & 0,27$\pm$0,06 \\ \hline
		$\{C_5\}$      & 0,44$\pm$0,02 & 0,34$\pm$0,02 & 0,23$\pm$0,01 & 0,42$\pm$0,01 & 0,27$\pm$0,03 \\ \hline
		$\{C_6\}$      & 0,31$\pm$0,00 & 0,27$\pm$0,02 & 0,25$\pm$0,01 & 0,30$\pm$0,01 & 0,26$\pm$0,09 \\ \hline
		$\{C_3, C_4\}$     & 0,33$\pm$0,01 & 0,27$\pm$0,01 & 0,24$\pm$0,02 & 0,32$\pm$0,01 & 0,23$\pm$0,03 \\ \hline
		$\{C_5, C_6\}$     & 0,28$\pm$0,01 & 0,26$\pm$0,01 & 0,23$\pm$0,01 & 0,28$\pm$0,01 & 0,20$\pm$0,03 \\ \hline
		$\{C_4, C_5, C_6\}$    & 0,24$\pm$0,00 & 0,21$\pm$0,00 & 0,20$\pm$0,00 & 0,25$\pm$0,01 & 0,16$\pm$0,02 \\ \hline
		$\{C_3, C_4, C_5, C_6\}$   & 0,23$\pm$0,00 & 0,21$\pm$0,00 & 0,20$\pm$0,01 & 0,26$\pm$0,02 & 0,18$\pm$0,02 \\ \hline
		$\{C_3,\ldots,C_{10}\}$ (hom)  & \textbf{0,22}$\pm$\textbf{0,01} & \textbf{0,20}$\pm$\textbf{0,00} & 0,19$\pm$0,00 & \textbf{0,2376}$\pm$\textbf{0,01} & \textbf{0,1352}$\pm$\textbf{0,01} \\ \hline
		$\{C_3, \ldots, C_{10}\}$ (iso)  & 0,24$\pm$0,01 & 0,22$\pm$0,01 & \textbf{0,16}$\pm$\textbf{0,01} & 0,2408$\pm$0,01 & 0,1357 $\pm$ 0,01\\ \hline
	\end{tabular}
	\end{center}\vspace{-2ex}
\end{table}

\begin{table}[hbt!]
	\caption{Total model parameters for selected cycle combinations  and $\GNNs$ on the ZINC data set. In the last two rows we compare between homomorphism counts (hom) and subgraph isomorphism counts (iso).}
	\label{ZINC_par}
	\begin{center}
	\begin{tabular}{|l|l|l|l|l|l|}
		\hline
		Pattern set $\mathcal{F}$ & $\mathsf{GAT}$           & $\mathsf{GCN}$           & $\mathsf{GraphSage}$     & $\mathsf{MoNet}$         & $\mathsf{GatedGCN}_{E,PE}$ \\ \hline\hline
	None     & 358\,273 & 360\,742 & 388\,963    & 401\,148 & 408\,135   \\ \hline
	$\{C_3\}$      & 358\,417 & 360\,887 & 389\,071    & 401\,238 & 408\,205   \\ \hline
	$\{C_4\}$      & 358\,417 & 360\,887 & 389\,071    & 401\,238 & 408\,205   \\ \hline
	$\{C_5\}$      & 358\,417 & 360\,887 & 389\,071    & 401\,238 & 408\,205   \\ \hline
	$\{C_6\}$      & 358\,417 & 360\,887 & 389\,071    & 401\,238 & 408\,205   \\ \hline
	$\{C_3, C_4\}$     & 358\,561 & 361\,032 & 389\,179    & 401\,328 & 408\,275   \\ \hline
	$\{C_5, C_6\}$     & 358\,561 & 361\,032 & 389\,179    & 401\,328 & 408\,275   \\ \hline
	$\{C_4, C_5, C_6\}$    & 358\,705 & 361\,177 & 389\,287    & 401\,418 & 408\,345   \\ \hline
	$\{C_3, C_4, C_5, C_6\}$   & 358\,849 & 361\,322 & 389\,395    & 401\,508 & 408\,415   \\ \hline
	$\{C_3,\ldots, C_{10}\}$ (hom) & 359\,425 & 361\,902 & 389\,827 & 401\,868 & 408\,695 \\ \hline
	$\{C_3, \ldots, C_{10}\}$ (iso)  & 359\,425 & 361\,902 & 389\,827 & 401\,868 & 408\,695 \\ \hline
	\end{tabular}
	\end{center}\vspace{-2ex}
\end{table}

\begin{table}[hbt!]
	\caption{Average training time in hours and number of epochs for selected cycle combinations and $\GNNs$ on the ZINC data set. In the last two rows we compare between homomorphism counts (hom) and subgraph isomorphism counts (iso).}
	\label{ZINC_time}
	\begin{center}
\adjustbox{max width=\textwidth}{\begin{tabular}{|l|l|l|l|l|l|l|l|l|l|l|}
		\hline
		Model:   & \multicolumn{2}{c|}{$\mathsf{GAT}$} & \multicolumn{2}{c|}{$\mathsf{GCN}$} & \multicolumn{2}{c|}{$\mathsf{GraphSage}$} & \multicolumn{2}{c|}{$\mathsf{MoNet}$} & \multicolumn{2}{c|}{$\mathsf{GatedGCN}_{E,PE}$} \\ \hline
		Pattern set $\mathcal{F}$ & Time      & Epochs     & Time      & Epochs     & Time         & Epochs        & Time       & Epochs      & Time           & Epochs          \\ \hline
		None     & 2,40      & 377          & 10,99     & 463          & 2,46         & 420             & 1,53       & 345           & 12,08          & 136               \\ \hline
		$\{C_3\}$      & 2,88      & 444          & 12,03     & 363          & 2,03         & 500             & 0,91       & 298           & 12,07          & 148               \\ \hline
		$\{C_4\}$      & 2,30      & 351          & 11,36     & 324          & 2,31         & 396             & 1,70       & 382           & 12,06          & 139               \\ \hline
		$\{C_5\}$      & 2,42      & 375          & 12,03     & 333          & 1,70         & 444             & 1,06       & 370           & 12,06          & 202               \\ \hline
		$\{C_6\}$      & 2,40      & 369          & 9,98      & 421          & 2,58         & 446             & 1,25       & 288           & 12,08          & 136               \\ \hline
		$\{C_3, C_4\}$     & 2,98      & 461          & 12,03     & 332          & 2,56         & 458             & 1,41       & 321           & 12,09          & 132               \\ \hline
		$\{C_5, C_6\}$     & 2,76      & 422          & 12,04     & 319          & 2,67         & 464             & 1,53       & 356           & 12,06          & 137               \\ \hline
		$\{C_4, C_5, C_6\}$    & 2,45      & 381          & 10,13     & 419          & 1,67         & 463             & 1,04       & 382           & 12,04          & 229               \\ \hline
		$\{C_3, C_4, C_5, C_6\}$   & 2,65      & 408          & 10,38     & 420          & 2,09         & 503             & 1,26       & 364           & 12,08          & 135               \\ \hline
		$\{C_3, \ldots, C_{10} \}$ (hom)  & 2,65      & 428          & 12,03     & 350          & 2,76         & 478             & 1,48       & 363           &   12,06        &      175         \\ \hline
		$\{C_3,\ldots, C_{10}\}$ (iso)   & 2,78      & 497          & 11,72     & 419          & 2,63         & 547            & 1,58       & 440           &   11,62        &   148             \\ \hline
	\end{tabular}}
	\end{center}
	\vspace{-2ex}
\end{table}

\subsubsection{Link Prediction with the Collab dataset}\label{subsec:prediction}
Another set used in \citet{dwivedi2020benchmarkgnns} is
COLLAB, a link prediction dataset proposed by the Open Graph Benchmark (OGB)  \citep{hu2020open} corresponding to a collaboration network
between approximately $235$K scientists, indexed by Microsoft Academic Graph. Vertices represent
scientists and edges denote collaborations between them. For vertex features, OGB provides $128$-dimensional vectors, obtained by averaging the word embeddings of a scientist’s papers. The year
and number of co-authored papers in a given year are concatenated to form edge features. The graph
can also be viewed as a dynamic multi-graph, since two vertices may have multiple temporal edges
between if they collaborate over multiple years. The following are taken from \citet{dwivedi2020benchmarkgnns}:

\noindent
\textbf{Splitting.} We use the real-life training, validation and test edge splits provided by OGB. Specifically,
they use collaborations until 2017 as training edges, those in 2018 as validation edges, and those in
2019 as test edges. \\
\textbf{Training.} All GNNs use the same learning rate strategy: an initial learning rate is set to $1 \times 10^{-3}$,
the reduce factor is $0.5$, the patience value is 10, and the stopping learning rate is $1 \times 10^{-5}$.
\\
\textbf{Performance Measure.} We use the evaluator provided by OGB \citep{hu2020open}, which aims to measure a model’s
ability to predict future collaboration relationships given past collaborations. Specifically, they rank
each true collaboration among a set of $100\,000$ randomly-sampled negative collaborations, and count
the ratio of positive edges that are ranked at $K$-place or above (Hits@K).  The value $K=50$ as this gives the best
value for statistically separating the performance of $\GNNs$.
\textbf{Number of layers}  3 MPNN layers are used for every model.

\begin{table}[hbt!]
	\caption{Full Results (Hits @50) for all selected pattern combinations and $\GNNs$ on the COLLAB data set.}
	\label{COLLAB_full}
	\begin{center}
	\begin{tabular}{|l|l|l|l|l|l|}
		\hline
		Pattern set $\mathcal{F}$ & $\mathsf{GAT}$           & $\mathsf{GCN}$           & $\mathsf{GraphSage}$     & $\mathsf{MoNet}$         & $\mathsf{GatedGCN}_{E,PE}$\\\hline\hline
	None     & 50,32$\pm$0,55 & 51,36$\pm$1,30 & 49,81$\pm$1,56 & 50,33$\pm$0,68 & 51,00$\pm$2,54 \\ \hline
	$\{K_3\}$      & \textbf{52,87}$\pm$\textbf{0,87} & 53,57$\pm$0,89 & 50,18$\pm$1,38 & 51,10$\pm$0,38 & \textbf{51,57}$\pm$\textbf{0,68} \\ \hline
	$\{ K_4\}$      & 51,33$\pm$1,42 & 52,84$\pm$1,32 & \textbf{51,76}$\pm$\textbf{1,38} & 51,13$\pm$1,60 & 49,43$\pm$1,85 \\ \hline
	$\{ K_5\}$      & 52,41$\pm$0,89 & \textbf{54,60}$\pm$\textbf{1,01} & 50,94$\pm$1,30 & \textbf{51,39}$\pm$\textbf{1,23} & 50,31$\pm$1,59 \\ \hline
	$\{K_3, K_4\}$     & 52,68$\pm$1,82 & 53,49$\pm$1,35 & 50,88$\pm$1,73 & 50,97$\pm$0,68 & 51,36$\pm$0,92 \\ \hline
	$\{K_3, K_4, K_5\}$    & 51,81$\pm$1,17 & 54,32$\pm$1,02 & 49,94$\pm$0,23 & 51,01$\pm$1,00 & 51,11$\pm$1,06 \\ \hline
	\end{tabular}
	\end{center}\vspace{-2ex}
\end{table}

\begin{table}[hbt!]
	\caption{Total number of model parameters for all selected pattern combinations and $\GNNs$ on the COLLAB data set.}
	\label{COLLAB_par}
	\begin{center}
	\begin{tabular}{|l|l|l|l|l|l|}
		\hline
		Pattern set $\mathcal{F}$ & $\mathsf{GAT}$           & $\mathsf{GCN}$           & $\mathsf{GraphSage}$     & $\mathsf{MoNet}$         & $\mathsf{GatedGCN}_{E,PE}$\\\hline\hline
		None     & 25\,992 & 40\,479 & 39\,751 & 26\,487     & 27\,440    \\ \hline
		$\{K_3\}$      & 26\,049 & 40\,553 & 39\,804 & 26\,525     & 27\,475    \\ \hline
		$\{ K_4\}$      & 26\,049 & 40\,553 & 39\,804 & 26\,525     & 27\,475    \\ \hline
		$\{ K_5\}$      & 26\,049 & 40\,553 & 39\,804 & 26\,525     & 27\,475    \\ \hline
		$\{K_3, K_4\}$     & 26\,106 & 40\,627 & 39\,857 & 26\,563     & 27\,510    \\ \hline
		$\{K_3, K_4, K_5\}$    & 26\,163 & 40\,701 & 39\,910 & 26\,601     & 27\,545    \\ \hline
	\end{tabular}
	\end{center}
\end{table}

\begin{table}[hbt!]
	\caption{Average training times and number of epochs for all selected pattern combinations and $\GNNs$ on the COLLAB data set.}
	\label{COLLAB_time}
	\begin{center}
\adjustbox{max width=\textwidth}{
	\begin{tabular}{|l|l|l|l|l|l|l|l|l|l|l|}
		\hline
		Model:   & \multicolumn{2}{c|}{$\mathsf{GAT}$} & \multicolumn{2}{c|}{$\mathsf{GCN}$} & \multicolumn{2}{c|}{$\mathsf{MoNet}$} & \multicolumn{2}{c|}{$\mathsf{GraphSage}$} & \multicolumn{2}{c|}{$\mathsf{GatedGCN}_{E,PE}$} \\ \hline\hline
		Pattern set $\mathcal{F}$ & Time      & \#Epochs     & Time      & \#Epochs     & Time       & \#Epochs      & Time         & \#Epochs        & Time           & \#Epochs          \\ \hline
		None     & 0,81      & 167          & 0,85      & 141          & 1,62       & 190           & 12,05       & 115,67           & 2,22           & 167               \\ \hline
		$\{K_3\}$      & 0,67      & 165          & 0,90      & 153          & 1,70       & 184           & 12,10        & 67,00           & 2,48           & 186               \\ \hline
		$\{ K_4\}$      & 1,06      & 188          & 0,95      & 160          & 2,16       & 188           & 12,04        & 113,50          & 1,26           & 188               \\ \hline
		$\{ K_5\}$      & 0,50      & 167          & 1,13      & 165          & 1,04       & 193           & 12,05        & 124,00          & 1,82           & 174               \\ \hline
		$\{K_3, K_4\}$     & 1,20      & 189          & 0,86      & 128          & 2,15       & 189           & 12,05        & 113,25          & 1,51           & 183               \\ \hline
		$\{K_3, K_4, K_5\}$    & 0,44      & 149          & 0,90      & 134          & 0,98       & 186           & 12,05        & 124,00          & 1,84           & 177               \\ \hline
	\end{tabular}}
	\end{center}
\end{table}

\subsubsection{Vertex classification with PATTERN and CLUSTER}\label{subsec:classification}
Finally, also used in \citet{dwivedi2020benchmarkgnns} are the PATTERN and CLUSTER graph data sets, generated with the Stochastic Block
Model (SBM) \citep{JMLR:v18:16-480}, which is widely used to model communities in social networks by modulating the
intra- and extra-communities connections, thereby controlling the difficulty of the task. A SBM is a
random graph which assigns communities to each vertex as follows: any two vertices are connected
with probability $p$ if they belong to the same community, or they are connected with probability
$q$ if they belong to different communities (the value of $q$ acts as the noise level).

For the PATTERN dataset, the goal of the vertex classification problem is the detection of a certain pattern $P$ embedded in a larger graph $G$. The graphs in  $G$ consist of
$5$ communities with sizes randomly selected between $[5, 35]$. The parameters of the SBM for each community is
$p = 0.5$, $q = 0.35$, and the vertex features in $G$ are generated using a uniform random distribution
with a vocabulary of size $3$, i.e., $\{0, 1, 2\}$. Randomly, $100$ patterns $P$ composed of $20$
vertices with intra-probability $p_P= 0.5$ and extra-probability $q_P= 0.5$ are generated (i.e., $50\%$ of vertices in $P$ are
connected to $G$). The vertex features for $P$ are also generated randomly using values in $\{0, 1, 2\}$.
The graphs consist of $44$-$188$ vertices. The output vertex labels have value $1$ if the vertex belongs to $P$
and value $0$ belongs to $G$.\\

For the CLUSTER dataset, the goal of the vertex classification is the detection of which cluster a vertex belongs. Here, six SBM clusters are generated with sizes randomly
selected between $[5, 35]$ and probabilities $p = 0.55$ and $q = 0.25$. The graphs consist of $40$-$190$ vertices.
Each vertex can take an initial feature value in range $\{0, 1, 2,\ldots, 6\}$. If the value is $i$ then the vertex
belongs to class $i-1$. If the value is $0$, then the class of the vertex is unknown and need to be inferred.
There is only one labelled vertex that is randomly assigned to each community and most vertex features are set to $0$. The output vertex labels are defined as the community/cluster class labels. 

\noindent
The following are taken from \citet{dwivedi2020benchmarkgnns}:

\noindent
\textbf{Splitting} The PATTERN dataset has $10\,000$ train, $2\,000$ validation and $2\,000$ test graphs. The CLUSTER
dataset has $10\,000$ train, $1\,000$ validation and $1\,000$ test graphs. We save the generated splits and use the
same sets in all models for fair comparison.\\
\textbf{Training}  For all $\GNNs$, an initial learning rate is set to $1 \times 10^{-3}$, the reduce
factor is $0.5$, the patience value is $10$, and the stopping learning rate is $1 \times 10^{-5}$
.\\
\textbf{Performance measure} The performance measure is the average vertex-level accuracy weighted with
respect to the class sizes.
\textbf{Number of layers}  16 MPNN layers are used for every model.
\vspace{-2ex}

\begin{table}[hbt!]
	\caption{Full results of the weighted accuracy for selected pattern combinations and $\GNNs$ on the CLUSTER data set.}
	\label{CLUSTER_full}
	\begin{center}\vspace{-1ex}
	\begin{tabular}{|l|l|l|l|l|l|}
		\hline
		Pattern set $\mathcal{F}$& $\mathsf{GAT}$            & $\mathsf{GCN}$             & $\mathsf{MoNet}$          & $\mathsf{GraphSage}$      & $\mathsf{GatedGCN}_{E,PE}$       \\ \hline\hline
	None     & 70,86$\pm$0,06 & \textbf{70,64}$\pm$\textbf{0,39}  & 71,15$\pm$0,33 & 72,25$\pm$0,52 & \textbf{74,28}$\pm$\textbf{0,15} \\ \hline
	$\{K_3\}$      & 71,60$\pm$0,15 & 64,88$\pm$4,16  & 72,21$\pm$0,19 & 72,97$\pm$0,23 & 74,14$\pm$0,12 \\ \hline
	$\{ K_4\}$      & 71,40$\pm$0,24 & 60,64$\pm$2,93  & 72,14$\pm$0,19 & 72,57$\pm$0,19 & 74,16$\pm$0,24 \\ \hline
	$\{ K_5\}$      & 71,26$\pm$0,39 & 66,60$\pm$1,47  & \textbf{72,34}$\pm$\textbf{0,09} & 72,60$\pm$0,24 & 74,23$\pm$0,07 \\ \hline
	$\{K_3, K_4\}$     & \textbf{71,80}$\pm$\textbf{0,28} & 50,94$\pm$22,98 & 72,32$\pm$0,27 & \textbf{73,03}$\pm$\textbf{0,25} & 74,17$\pm$0,13 \\ \hline
	$\{K_3, K_4, K_5\}$    & 71,63$\pm$0,26 & 63,03$\pm$3,72  & 72,32$\pm$0,36 & 72,65$\pm$0,13 & 74,03$\pm$0,19 \\ \hline
	\end{tabular}
\end{center}
\end{table}
\vspace{-1ex}
\begin{table}
	\caption{Total number of model parameters  for all selected pattern combinations and $\GNNs$ on the CLUSTER data set.}
	\label{CLUSTER_par}
	\begin{center}\vspace{-1ex}
	\begin{tabular}{|l|l|l|l|l|l|}
		\hline
		Pattern set $\mathcal{F}$ & $\mathsf{GAT}$    & $\mathsf{GCN}$    & $\mathsf{MoNet}$  & $\mathsf{GraphSage}$ & $\mathsf{GatedGCN}_{E,PE}$ \\ \hline\hline
		None     & 395\,396 & 362\,849 & 399\,373 & 386\,835    & 406\,755   \\ \hline
		None     & 395\,396 & 362\,849 & 399\,373 & 386\,835    & 406\,755   \\ \hline
		$\{K_3\}$      & 395\,396 & 362\,849 & 399\,373 & 386\,835    & 406\,755   \\ \hline
		$\{ K_4\}$      & 395\,548 & 362\,995 & 399\,463 & 386\,943    & 406\,825   \\ \hline
		$\{ K_5\}$      & 395\,700 & 363\,141 & 399\,553 & 387\,051    & 406\,895   \\ \hline
		$\{K_3, K_4\}$     & 395\,700 & 363\,141 & 399\,553 & 387\,051    & 406\,895   \\ \hline
		$\{K_3, K_4, K_5\}$    & 396\,004 & 363\,433 & 399\,733 & 387\,267    & 407\,035   \\ \hline
	\end{tabular}
	\end{center}
\end{table}

\begin{table}[hbt!]
	\caption{Training times (in hours) and number of epochs for all selected pattern combinations and $\GNNs$ on the CLUSTER data set.}
	\label{CLUSTER_time}
	\begin{center}
\adjustbox{max width=\textwidth}{	\begin{tabular}{|l|l|l|l|l|l|l|l|l|l|l|}
		\hline
		Model:   & \multicolumn{2}{c|}{GAT} & \multicolumn{2}{c|}{$\mathsf{GCN}$} & \multicolumn{2}{c|}{$\mathsf{MoNet}$} & \multicolumn{2}{c|}{$\mathsf{GraphSage}$} & \multicolumn{2}{c|}{$\mathsf{GatedGCN}_{E,PE}$} \\ \hline\hline
		Pattern set $\mathcal{F}$ & Time      & \#Epochs     & Time      & \#Epochs     & Time       & \#Epochs      & Time         & \#Epochs        & Time             & \#Epochs            \\ \hline
		None     & 1,62      & 109          & 2,83      & 117          & 1,54       & 125           & 0,95         & 101             & 10,40            & 92                  \\ \hline
		$\{K_3\}$      & 1,52      & 107          & 2,67      & 85           & 1,72       & 145           & 1,08         & 102             & 11,01            & 89                  \\ \hline
		$\{ K_4\}$      & 1,18      & 107          & 1,94      & 80           & 1,62       & 149           & 0,90         & 102             & 10,23            & 90                  \\ \hline
		$\{ K_5\}$      & 1,23      & 106          & 2,30      & 84           & 1,68       & 143           & 0,92         & 99              & 10,68            & 91                  \\ \hline
		$\{K_3, K_4\}$     & 1,53      & 102          & 1,97      & 82           & 1,89       & 153           & 0,94         & 99              & 10,80            & 90                  \\ \hline
		$\{K_3, K_4, K_5\}$    & 1,62      & 105          & 1,96      & 82           & 1,95       & 157           & 0,97         & 100             & 10,25            & 91                  \\ \hline
	\end{tabular}}
	\end{center}
\end{table}

\begin{table}[hbt!]
	\caption{Full results of the weighted accuracy for selected pattern combinations and $\GNNs$ on the PATTERN data set.}
		\label{PATTERN_full}
		\begin{center}
		\begin{tabular}{|l|l|l|l|l|l|}
			\hline
			Pattern set $\mathcal{F}$ & $\mathsf{GAT}$            & $\mathsf{GCN}$            & $\mathsf{MoNet}$          & $\mathsf{GraphSage}$      & $\mathsf{GatedGCN}_{E,PE}$       \\ \hline\hline
			None     & 78,83$\pm$0,60 & 71,42$\pm$1,38 & 85,90$\pm$0,03 & 70,78$\pm$0,19 & \textbf{86,15}$\pm$\textbf{0,08} \\ \hline
			$\{K_3\}$      & 84,34$\pm$0,09 & 61,54$\pm$2,20 & 86,59$\pm$0,02 & 84,75$\pm$0,11 & 85,02$\pm$0,20 \\ \hline
			$\{ K_4\}$      & 84,43$\pm$0,40 & 63,40$\pm$1,55 & 86,60$\pm$0,02 & 84,51$\pm$0,06 & 85,40$\pm$0,28 \\ \hline
			$\{ K_5\}$      & 83,47$\pm$0,11 & 64,18$\pm$3,88 & 86,57$\pm$0,02 & 83,73$\pm$0,10 & 85,63$\pm$0,22 \\ \hline
			$\{K_3, K_4\}$     & 85,44$\pm$0,24 & 81,29$\pm$2,82 & 86,58$\pm$0,02 & 85,85$\pm$0,13 & 85,80$\pm$0,20 \\ \hline
			$\{K_3, K_4, K_5\}$    & \textbf{85,50}$\pm$\textbf{0,23} & \textbf{82,49}$\pm$\textbf{0,48} & \textbf{86,63}$\pm$\textbf{0,03} & \textbf{85,88}$\pm$\textbf{0,15} & 85,56$\pm$0,33 \\ \hline
		\end{tabular}
		\end{center}
\end{table}

\begin{table}[hbt!]
	\caption{Total number of model parameters for selected pattern combinations and $\GNNs$ on the PATTERN data set.}
	\label{PATTERN_par}
	\begin{center}
	\begin{tabular}{|l|l|l|l|l|l|}
		\hline
		Pattern set $\mathcal{F}$ & $\mathsf{GAT}$    & $\mathsf{GCN}$    & $\mathsf{MoNet}$  & $\mathsf{GraphSage}$ & $\mathsf{GatedGCN}_{E,PE}$ \\ \hline\hline
	None     & 394\,632 & 362\,117 & 398\,921 & 386\,291    & 406\,403   \\ \hline
	$\{K_3\}$      & 394\,784 & 362\,263 & 399\,011 & 386\,399    & 406\,473   \\ \hline
	$\{ K_4\}$      & 394\,784 & 362\,263 & 399\,011 & 386\,399    & 406\,473   \\ \hline
	$\{ K_5\}$      & 394\,784 & 362\,263 & 399\,011 & 386\,399    & 406\,473   \\ \hline
	$\{K_3, K_4\}$     & 394\,936 & 362\,409 & 399\,101 & 386\,507    & 406\,543   \\ \hline
	$\{K_3, K_4, K_5\}$    & 395\,088 & 362\,555 & 399\,191 & 386\,615    & 406\,613   \\ \hline
	\end{tabular}
	\end{center}
\end{table}

\begin{table}[hbt!]
	\caption{Training times (in hours) and number of epochs for selected pattern combinations and $\GNNs$ on the PATTERN data set.}
	\label{PATTERN_time}
	\begin{center}	
\adjustbox{max width=\textwidth}{	\begin{tabular}{|l|l|l|l|l|l|l|l|l|l|l|}
		\hline
		Model:   & \multicolumn{2}{c|}{GAT} & \multicolumn{2}{c|}{$\mathsf{GCN}$} & \multicolumn{2}{c|}{$\mathsf{MoNet}$} & \multicolumn{2}{c|}{$\mathsf{GraphSage}$} & \multicolumn{2}{c|}{$\mathsf{GatedGCN}_{E,PE}$} \\ \hline\hline
		Pattern set $\mathcal{F}$ & Time       & Epochs      & Time       & Epochs      & Time        & Epochs       & Time          & Epochs         & Time            & Epochs           \\ \hline
		None     & 1,96       & 87          & 3,41       & 102         & 1,68        & 116          & 0,77          & 103            & 10,32           & 101              \\ \hline
		$\{K_3\}$      & 0,97       & 97          & 2,58       & 80          & 1,42        & 107          & 0,69          & 105            & 9,12            & 95               \\ \hline
		$\{ K_4\}$      & 0,90       & 90          & 2,68       & 80          & 1,46        & 106          & 0,67          & 95             & 9,47            & 94               \\ \hline
		$\{ K_5\}$      & 0,89       & 95          & 2,36       & 80          & 1,26        & 100          & 0,58          & 98             & 9,14            & 99               \\ \hline
		$\{K_3, K_4\}$     & 2,11       & 91          & 3,62       & 98          & 1,68        & 108          & 0,86          & 97             & 9,50            & 87               \\ \hline
		$\{K_3, K_4, K_5\}$    & 1,02       & 91          & 3,26       & 94          & 1,48        & 109          & 0,76          & 102            & 8,84            & 88               \\ \hline
	\end{tabular}}
	\end{center}
\end{table}

\end{document}